\documentclass[journal,12pt,onecolumn] {IEEEtran}

\ifCLASSINFOpdf
\usepackage[pdftex]{graphicx}
\DeclareGraphicsExtensions{.pdf,.jpeg,.png}
\else
\usepackage[dvips]{graphicx}
\DeclareGraphicsExtensions{.eps}
\fi

\usepackage[cmex10]{amsmath}
\usepackage[caption=false,font=footnotesize]{subfig}
\usepackage{amssymb}
\usepackage{amsthm}
\usepackage{gensymb}
\usepackage{flushend}

\newtheorem{thm}{Theorem}
\newtheorem{cor}{Corollary}
\newtheorem{lem}{Lemma}
\newtheorem{defn}{Definition}

\begin{document}
	
\title{Optimal Clustering under Uncertainty}
\author{Lori~A.~Dalton, Marco~E.~Benalc\'{a}zar, and~Edward~R.~Dougherty
	
\thanks{L.~A.~Dalton is with the Department of Electrical and Computer Engineering, The Ohio State University, Columbus, OH 43210 USA (e-mail: dalton.lori@outlook.com).} 
\thanks{M.~E.~Benalc\'{a}zar is with the Secretar\'{\i}a Nacional de Educaci\'{o}n Superior, Ciencia, Tecnolog\'{\i}a e Innovaci\'{o}n (SENESCYT), Ecuador, the Consejo Nacional de Investigaciones Cient\'{\i}ficas y T\'{e}cnicas (CONICET), Argentina, the Facultad de Ingenier\'{\i}a, Universidad Nacional de Mar del Plata, Mar del Plata, Argentina, and the Escuela Polit\'{e}cnica Nacional, Departamento de Inform\'{a}tica y Ciencias de la Computaci\'{o}n, Quito, Ecuador (e-mail: marco\_benalcazar@hotmail.com).} 
\thanks{E.~R.~Dougherty is with the Department of Electrical and Computer Engineering, Texas A\&M University, College Station, TX 77843 USA (e-mail: edward@ece.tamu.edu).}}

\maketitle
		
\begin{abstract}
	Classical clustering algorithms typically either lack an underlying probability framework to make them predictive or focus on parameter estimation rather than defining and minimizing a notion of error. Recent work addresses these issues by developing a probabilistic framework based on the theory of random labeled point processes and characterizing a \emph{Bayes clusterer} that minimizes the number of misclustered points. The Bayes clusterer is analogous to the Bayes classifier. Whereas determining a Bayes classifier requires full knowledge of the feature-label distribution, deriving a Bayes clusterer requires full knowledge of the point process. When uncertain of the point process, one would like to find a robust clusterer that is optimal over the uncertainty, just as one may find optimal robust classifiers with uncertain feature-label distributions. Herein, we derive an optimal robust clusterer by first finding an \emph{effective} random point process that incorporates all randomness within its own probabilistic structure and from which a Bayes clusterer can be derived that provides an optimal robust clusterer relative to the uncertainty. This is analogous to the use of effective class-conditional distributions in robust classification. After evaluating the performance of robust clusterers in synthetic mixtures of Gaussians models, we apply the framework to granular imaging, where we make use of the asymptotic granulometric moment theory for granular images to relate robust clustering theory to the application.
\end{abstract}
		

\section{Introduction}

The basic optimization paradigm for operator design consists of four parts:
(1) define the underlying random process; (2) define the class of potential
operators; (3) characterize operator performance via a cost function; and
(4) find an operator to minimize the cost function. The classic example is
the Wiener filter, where the four parts consist of wide-sense stationary
true and observed signals, linear operators, minimization of the mean-square
error, and optimization in terms of power spectra. In practice, we might be
uncertain as to the distribution governing the random process so that we
desire a \emph{robust} operator, one whose performance is acceptable
relative to the uncertainty. Robust design can be posed in the following
way: Given a class of operators and given that the \emph{state of nature} is
uncertain but contained in some \emph{uncertainty class}, which operator
should be selected to optimize performance across all possible states of
nature?

Our interest here is in clustering, where the underlying process is a random
point set and the aim is to partition the point set into clusters
corresponding to the manner in which the points have been generated by the
underlying process. Having developed the theory of optimal clustering in the
context of random labeled point sets where optimality is with respect to
mis-clustered points~\cite{dalton2015analytic}, we now consider optimal
clustering when the underlying random labeled point process belongs to an
uncertainty class of random labeled point processes, so that optimization is
relative to both clustering error and model uncertainty. This is analogous
to finding an optimal Wiener filter when the signal process is unknown, so
that the power spectra belong to an uncertainty class~\cite%
{dalton2014intrinsically}. We now briefly review classical robust operator
theory, which will serve as the foundation for a new general theory of
optimal robust clustering.

Optimal robust filtering first appeared in signal processing in the 1970s
when the problem was addressed for signals with uncertain power spectra.
Early work considered robust filter design from a minimax perspective: the
filter is designed for the state having the best worst-case performance over
all states \cite{Kuznetsov:76,Kassam:77,Poor:80}. Whereas the standard
optimization problem given certainty with regard to the random process takes
the form 
\begin{equation}
\psi ^{\ast }=\arg \min_{\psi \in \mathcal{C}}\gamma (\psi ),
\label{classical}
\end{equation}%
where $\mathcal{C}$ is the operator class and $\gamma (\psi )$ is the cost
of applying operator $\psi $ on the model, minimax optimization is defined
by 
\begin{equation}
\psi _{\mathrm{MM}}=\arg \min_{\psi \in \mathcal{C}_{\Theta }}\max_{\theta
\in \Theta }\gamma _{\theta }(\psi ),  \label{MC_minimax}
\end{equation}%
where $\Theta $ is the uncertainty class of random processes, $\mathcal{C}%
_{\Theta }$ is the class of operators that are optimal for some state in the
uncertainty class, and $\gamma _{\theta }(\psi )$ is the cost of applying
operator $\psi $ for state $\theta \in \Theta $.

Suppose one has prior knowledge with which to construct a prior distribution 
$\pi (\theta )$ on states (models) in the uncertainty class. Rather than
apply a minimax robust operator, whose average performance can be poor, a
Bayesian approach can be taken whereby optimization is relative to $\pi
(\theta )$. A \textit{model-constrained
(state-constrained) Bayesian robust (MCBR) operator }minimizes the expected
error over the uncertainty class among all operators in $\mathcal{C}_{\Theta
}$: 
\begin{equation}
\psi _{\mathrm{MCBR}}=\arg \min_{\psi \in \mathcal{C}_{\Theta }}E_{\theta
}[\gamma _{\theta }(\psi )].  \label{MCBR}
\end{equation}%
MCBR filtering has been considered for morphological,
binary and linear filtering. MCBR
design has also been applied in classification with uncertain feature-label
distributions~\cite{dougherty2005optimal}.

Rather than restrict optimization to operators that are optimal for some
state in the uncertainty class, one can optimize over any class of
operators, including unconstrained optimization over all possible
measurable functions. In this case, the optimal operator is called an \emph{%
intrinsically optimal Bayesian robust (IBR) operator (filter)} and~%
\eqref{MCBR} becomes 
\begin{equation}
\psi _{\mathrm{IBR}}=\arg \min_{\psi \in \mathcal{C}}E_{\theta }[\gamma
_{\theta }(\psi )],  \label{IBR}
\end{equation}%
where $\mathcal{C}$ is a set of operators under consideration. IBR filtering
has been considered for linear and morphological filtering \cite%
{dalton2014intrinsically}. The IBR approach was first used to design optimal
classifiers when the unknown true feature-label distribution belongs to an
uncertainty class~\cite{dalton2013optimalI,dalton2013optimalII}. In that
setting, optimization is relative to a posterior distribution obtained from
the prior utilizing sample data and an optimal classifier is called an \emph{%
optimal Bayesian classifier (OBC)}.

Unlike the state of affairs in filtering and classification, classical
clustering algorithms typically lack an underlying probability framework to
make them predictive. The exceptions, for instance, expectation-maximization
based on mixture models, typically focus on parameter estimation rather than
defining and minimizing a notion of operator error. Work in~\cite%
{DoughBrun:04} and~\cite{dalton2015analytic} addresses the solution to~%
\eqref{classical} in the context of clustering using a probabilistic theory
of clustering for random labeled point sets and a definition of clustering
error given by the expected number of \textquotedblleft
misclustered\textquotedblright\ points. This results in a \emph{Bayes
clusterer}, which minimizes error under the assumed probabilistic framework.
An (optimal) Bayes clusterer is analogous to an (optimal) \emph{Bayes}
classifier, which minimizes classification error under the assumed
feature-label distribution. Here, we characterize robust clustering using
the framework and definitions of error in~\cite{DoughBrun:04} and~\cite%
{dalton2015analytic}, and introduce definitions of robust clustering that
parallel concepts from filtering. In particular, we present minimax, MCBR
and IBR clusterers, and develop effective stochastic processes for robust
clustering. We also evaluate performance under mixtures of Gaussians and
demonstrate how the methodology can be implemented in practice with an
example from granular imaging.

\section{Bayes Clustering Theory}

In this section, we review Bayes clustering theory from~\cite{DoughBrun:04}
and~\cite{dalton2015analytic}. A \textit{random labeled point process}
(RLPP) is characterized by a pair, $(\Xi ,\Lambda )$, where $\Xi $ is a
point process generating a point set $S\subset \mathbb{R}^{d}$ and $\Lambda $
generates random labels on the points in $S$. In particular, let $\eta (S)$
denote the number of points in $S$. The first component in this pair, $\Xi $%
, maps from a probability space to $(\mathbf{N},\mathcal{N})$, where $%
\mathbf{N}$ is the family of finite sequences in $\mathbb{R}^{d}$ and $%
\mathcal{N}$ is the smallest $\sigma $-algebra on $\mathbf{N}$ such that for
any Borel set $B$ in $\mathbb{R}^{d}$ the mapping $S\mapsto \eta (S\cap B)$
is measurable. A probability measure, $\nu $, of $\Xi $ is determined by the
probabilities $\nu (Y)$ for $Y\in \mathcal{N}$, or (via the
Choquet-Matheron-Kendall theorem~\cite{choquet1954theory,
kendall1974foundations, matheron1975random, chiu2013stochastic}), may be
reduced to the system of probabilities $P(\Xi \cap K\neq \emptyset )$ over
all compact sets $K\subseteq \mathbb{R}^{d}$. Given a point set $S\in 
\mathbf{N}$, a label function $\phi _{S}:S\rightarrow L=\{1,2,\ldots ,l\}$
is a deterministic mapping that assigns each point $\mathbf{x}\in S$ to
label $\phi _{S}(\mathbf{x})\in L$. The second component, $\Lambda $, is a
random labeling, that is, $\Lambda =\{\Phi _{S}:S\in \mathbf{N}\}$, where $%
\Phi _{S}$ is a random label function with probability mass $P(\Phi
_{S}=\phi _{S}|S)$ on $L^{S}$.

For any set $S$, and pair of label functions $\phi _{S}$ and $\varphi _{S}$,
define the \textit{label mismatch error} between $\phi _{S}$ and $\varphi
_{S}$ to be the proportion of points where the label functions differ: 
\begin{equation}
\varepsilon (S,\phi _{S},\varphi _{S})=\frac{1}{\eta (S)}\sum_{\mathbf{x}\in
S}I_{\phi _{S}(\mathbf{x})\neq \varphi _{S}(\mathbf{x})},
\label{eq:label_mismatch_error}
\end{equation}
where $I_{A}$ is an indicator function equal to 1 if $A$ is true and 0
otherwise. Clustering involves identifying partitions of a point set rather
than the actual labeling. A partition of $S$ into $l$ clusters has the form $%
\mathcal{P}_{S}=\{S_{1},S_{2},\ldots ,S_{l}\}$ such that the $S_{y}$ are
disjoint and $S=\bigcup_{y=1}^{l}S_{y}$. Every partition $\mathcal{P}_{S}$
has associated with it a family, $G_{\mathcal{P}_{S}}$, of label functions
that induce the partition $\mathcal{P}_{S}$. That is, $\varphi _{S}\in G_{%
\mathcal{P}_{S}}$ if and only if $\mathcal{P}_{S}=\{S_{1},S_{2},\ldots
,S_{l}\}$ where $S_{y}=\{\mathbf{x}\in S:\varphi _{S}(\mathbf{x})=\ell
_{y}\} $ and $(\ell _{1},\ldots ,\ell _{l})$ is a permutation of $L$. For
any point set $S$, label function $\phi _{S}$, and partition $\mathcal{P}_{S}$%
, define the \textit{cluster mismatch error} to be the minimum label
mismatch error between $\phi _{S}$ and all label functions that induce $%
\mathcal{P}_{S}$: 
\begin{equation}
\varepsilon (S,\phi _{S},\mathcal{P}_{S})=\min_{\varphi _{S}\in G_{\mathcal{P%
}_{S}}}\varepsilon (S,\phi _{S},\varphi _{S}).  \label{eq:mismatch_error}
\end{equation}
This is a simplified version of the original definition in~\cite%
{DoughBrun:04}. Define the \textit{partition error} of $\mathcal{P}_{S}$ to
be the mean cluster mismatch error over the distribution of label functions
on $S$: 
\begin{align}
\varepsilon (S,\mathcal{P}_{S})& =E_{\Phi _{S}}[\varepsilon (S,\Phi _{S},%
\mathcal{P}_{S})|S]  \notag \\
& =E_{\Phi _{S}}\left[ \left. \min_{\varphi _{S}\in G_{\mathcal{P}%
_{S}}}\varepsilon (S,\Phi _{S},\varphi _{S})\right\vert S\right] .
\label{eq:error_of_partition}
\end{align}

In~\cite{dalton2015analytic}, it was shown that~\eqref{eq:error_of_partition}
can be written in the form 
\begin{equation}
\varepsilon (S,\mathcal{P}_{S})=\sum_{\mathcal{Q}_{S}\in \mathcal{K}%
_{S}}c_{S}(\mathcal{P}_{S},\mathcal{Q}_{S})P_{S}(\mathcal{Q}_{S}),
\label{eq:error_of_partition:final}
\end{equation}
where $\mathcal{K}_{S}$ is the set of all partitions of $S$, 
\begin{equation}
P_{S}(\mathcal{Q}_{S})=\sum_{\phi _{S}\in G_{\mathcal{Q}_{S}}}P(\Phi
_{S}=\phi _{S}|S)  \label{eq:partition_probability}
\end{equation}
is the probability mass function on partitions $\mathcal{Q}_{S}\in \mathcal{K%
}_{S}$ of $S$, and we define the \textit{natural partition cost function}, 
\begin{equation}
c_{S}(\mathcal{P}_{S},\mathcal{Q}_{S})=\frac{1}{\eta (S)}\min_{\varphi
_{S}\in G_{\mathcal{P}_{S}},\phi _{S}\in G_{\mathcal{Q}_{S}}}\sum_{\mathbf{x}
\in S}I_{\phi _{S}(\mathbf{x})\neq \varphi _{S}(\mathbf{x})}.
\label{eq:cost_function:final}
\end{equation}
The partition error under the natural cost function is essentially the
average number of misclustered points.

Taking~\eqref{eq:error_of_partition:final} as a generalized definition,
other cost functions can be applied~\cite{binder1978bayesian,
quintana2003bayesian, fritsch2009improved, meilua2007comparing}. The natural
cost function stands out in two respects. First, while these works define
loss over label functions, we define cost directly over partitions, which is
mathematically cleaner, and automatically treats the label switching problem in which multiple distinct label functions may produce the same partitions.
Second, these works treat loss abstractly without connecting to a practical
notion of clustering error, like the expected (minimum) number of mislabeled
points. In contrast, we begin with a practical definition of clustering
error, and show that minimizing clustering error equivalently minimizes~%
\eqref{eq:error_of_partition:final} relative to the natural cost function.

Let $\mathcal{C}_{S}=\{\mathcal{P}_{S}^{1},\ldots ,\mathcal{P}%
_{S}^{c}\}\subseteq \mathcal{K}_{S}$ be a set of $c$ \emph{candidate}
partitions that comprise the search space and $\mathcal{R}_{S}=\{\mathcal{Q}%
_{S}^{1},\ldots ,\mathcal{Q}_{S}^{r}\}\subseteq \mathcal{K}_{S}$ be a set of 
$r$ \emph{reference} partitions with known probabilities. The partition
error of all candidate partitions may be found via 
\begin{align}
& 
\begin{bmatrix}
\varepsilon (S,\mathcal{P}_{S}^{1}) & \cdots & \varepsilon (S,\mathcal{P}%
_{S}^{c})%
\end{bmatrix}%
^{T}  \notag \\
& =%
\begin{bmatrix}
c_{S}(\mathcal{Q}_{S}^{1},\mathcal{P}_{S}^{1}) & \cdots & c_{S}(\mathcal{Q}%
_{S}^{r},\mathcal{P}_{S}^{1}) \\ 
\vdots & \ddots & \vdots \\ 
c_{S}(\mathcal{Q}_{S}^{1},\mathcal{P}_{S}^{c}) & \cdots & c_{S}(\mathcal{Q}%
_{S}^{r},\mathcal{P}_{S}^{c})%
\end{bmatrix}%
\begin{bmatrix}
P_{S}(\mathcal{Q}_{S}^{1}) \\ 
\vdots \\ 
P_{S}(\mathcal{Q}_{S}^{r})%
\end{bmatrix}%
,  \label{eq:error_of_partition:vectors}
\end{align}%
where $T$ denotes matrix transpose. Given $S$, setting $\mathcal{C}_{S}=%
\mathcal{R}_{S}=\mathcal{K}_{S}$ requires a cost matrix of size $|\mathcal{K}%
_{S}|\times |\mathcal{K}_{S}|$, which can be prohibitively large for
moderate $\eta (S)$. To alleviate this, \cite{dalton2015analytic} provides
both exact and approximate techniques to evaluate~%
\eqref{eq:error_of_partition:vectors} under the natural cost function with
reduced complexity.

A cluster operator $\zeta $ maps point sets to partitions. Define the 
\textit{clustering error} of cluster operator $\zeta $ to be the mean
partition error of $\zeta (\Xi )$ over the random point sets $\Xi $: 
\begin{equation*}
\varepsilon (\zeta )=E_{\Xi }[\varepsilon (\Xi ,\zeta (\Xi ))].
\end{equation*}
A \emph{Bayes cluster operator} $\zeta ^{\ast }$ is a clusterer having
minimal clustering error $\varepsilon (\zeta ^{\ast })$, which is called the 
\emph{Bayes clustering error}. Since $\varepsilon (\zeta )=E_{\Xi
}[\varepsilon (\Xi ,\zeta (\Xi ))]$ and $\varepsilon (S,\zeta (S))$ depends
on the clusterer $\zeta $ only at point set $S$, $\varepsilon (\zeta )$ is
minimized by setting $\zeta ^{\ast }(S)=\mathcal{P}_{S}^{\ast }$ for all $%
S\in \mathbf{N}$, where $\mathcal{P}_{S}^{\ast }$ is a \emph{Bayes partition}
of $S$, defined to be a partition having minimal partition error, $%
\varepsilon (S,\mathcal{P}_{S}^{\ast })$, called the \emph{Bayes partition
error}.

This formulation parallels classification theory, where an RLPP corresponds
to a feature-label distribution, $\varepsilon (S,\mathcal{P}_{S})$
corresponds to the probability that a given label is incorrect for a fixed
point in the feature space, $\varepsilon (\zeta )$ corresponds to the
overall classification error for an arbitrary classifier, $\zeta ^{\ast }$
corresponds to a Bayes classifier, and $\varepsilon (\zeta ^{\ast })$
corresponds to the Bayes classification error.

\subsection{Separable RLPPs}

Up to this point, we have characterized RLPPs with a point process $\Xi$
that generates point sets, $S$, followed by an $S$-conditioned labeling
process $\Lambda $ that generates label functions, $\phi _{S}$.
Alternatively, it is often easier to characterize an RLPP as a process that
draws a sample size $n$, a set of labels for $n$ points, and a set of $n$
points with distributions corresponding to the labels. For instance, one
might think of points being drawn from $l$ Gaussian distributions possessing
random parameters. We say that an RLPP is \emph{separable} if a label
function $\phi $ is generated from an independent label generating process $%
\Phi $ with probability mass function $P(\Phi =\phi )$ over the set of all
label functions with domain $\{1,2,\ldots ,n\}$, a random parameter vector $%
\rho $ is independently drawn from a distribution $f(\rho )$, and the $i$th
point $\mathbf{x}_{i}$ in $S$, with corresponding label $y=\phi (i)$, is
independently drawn from a conditional distribution $f(\mathbf{x}|y,\rho )$.
From Bayes rule, the probability of label function $\phi _{S}\in L^{S}$
given $S=\{\mathbf{x}_{1},\ldots ,\mathbf{x}_{n}\}$ is 
\begin{equation}
P(\Phi _{S}=\phi _{S}|S) \propto f(S | \phi) P(\Phi =\phi ),
\label{eq:label_probability_main}
\end{equation}
where $\phi (i)=\phi _{S}(\mathbf{x}_{i})$, 
\begin{equation}
f(S | \phi) = \int \Bigg( \prod_{y=1}^{l}\prod_{\mathbf{x}\in S_{y}}f(%
\mathbf{x}|y,\rho ) \Bigg) f(\rho) d\rho,  \label{eq:label_probability_joint}
\end{equation}
and $S_y = \{\mathbf{x}_i:\phi(i) = y, i = 1, \ldots, n\}$ is the set of
points in $S$ assigned label $y$. A separable RLPP thus has three
components: $P(\Phi = \phi)$, $f(\rho)$ and $f(\mathbf{x} | y, \rho)$, where 
$P(\Phi = \phi)$ is a prior on labels, which is not dependent on $S$, and $%
P(\Phi_S = \phi_S | S)$ is a posterior probability on labels given a
specific point set $S$, which is found using~%
\eqref{eq:label_probability_main} and~\eqref{eq:label_probability_joint}.

If $\rho =[\rho _{1},\ldots ,\rho _{l}]$, where the $\rho _{y}$ are mutually
independent parameter vectors and the label-$y$-conditional distribution
depends on only $\rho _{y}$, that is, if $f(\mathbf{x}|y,\rho )=f(\mathbf{x}%
|y,\rho _{y})$ for $y=1,\ldots ,l$, then, 
\begin{equation}
f(S | \phi) = \prod_{y=1}^{l}\int \Bigg(\prod_{\mathbf{x}\in S_{y}}f(\mathbf{%
x}|y,\rho _{y})\Bigg) f(\rho_{y})d\rho _{y}.  \label{eq:label_probability}
\end{equation}

\subsection{Gaussian RLPPs}

\label{sec:Gaussian_RLPPs}

Expressions for label function probabilities have been solved under several
models in~\cite{dalton2015analytic}. Here, we review an important case in
which clusters are Gaussian with random means and covariances. Specifically,
consider a separable RLPP where, for each $y\in \{1,\ldots ,l\}$, $\rho
_{y}=[\mu _{y},\Sigma _{y}]$ and $f(\mathbf{x}|y,\rho _{y})$ is a Gaussian
distribution with mean $\mu _{y}$ and covariance $\Sigma _{y}$. Given a
label function $\phi _{S}$, let $y\in \{1,\ldots ,l\}$ be fixed, and let $%
n_{y}$ be the number of points in $S$ assigned label $y$. For $n_{y}\geq 2$
it was shown in~\cite{dalton2015analytic} that 
\begin{equation}
\prod_{\mathbf{x}\in S_{y}}f(\mathbf{x}|y,\rho _{y})=(2\pi )^{-\frac{dn_{y}}{%
2}}|\Sigma _{y}|^{-\frac{n_{y}}{2}}\exp \left( -\frac{1}{2}\mathrm{tr}\left(
\Phi _{y}^{\ast }\Sigma _{y}^{-1}\right) \right) ,
\label{eq:gaussian_model0}
\end{equation}
where $|\cdot |$ is a determinant, $\mathrm{tr}(\cdot )$ is a trace, 
\begin{equation*}
\Phi _{y}^{\ast }=(n_{y}-1)\widehat{\Sigma }_{y}+n_{y}\left( \mu _{y}-%
\widehat{\mu }_{y}\right) \left( \mu _{y}-\widehat{\mu }_{y}\right) ^{T},
\end{equation*}
and $\widehat{\mu }_{y}$ and $\widehat{\Sigma }_{y}$ are the sample mean and
covariance of points in $S_{y}$, respectively. When $n_{y}=1$, %
\eqref{eq:gaussian_model0} holds with $\Phi _{y}^{\ast }=\left( \mu _{y}-%
\widehat{\mu }_{y}\right) \left( \mu _{y}-\widehat{\mu }_{y}\right) ^{T}$,
and when $n_{y}=0$ the product over an empty set is $1$.

Assume $f(\rho_y) = f(\Sigma_y) f(\mu_y | \Sigma_y)$, where $f(\mu_y |
\Sigma_y)$ is a Gaussian distribution with mean $\mathbf{m}_{y}$ and
covariance $\frac{1}{\nu_y}\Sigma _{y}$ with $\nu_y > 0$, and $f(\Sigma
_{y}) $ is an inverse-Wishart distribution with $\kappa_y > d-1$ degrees of
freedom and a positive-definite scale matrix $\Psi _{y}$, i.e., 
\begin{equation*}
f(\Sigma _{y})=\frac{|\Psi _{y}|^{\frac{\kappa _{y}}{2}} |\Sigma _{y}|^{-%
\frac{\kappa_{y}+d+1}{2}}} {2^{\frac{\kappa_{y}d}{2}}\Gamma _{d}(\frac{%
\kappa _{y}}{2})} \exp \left( -\frac{1}{2}\mathrm{tr}(\Psi
_{y}\Sigma_{y}^{-1})\right),
\end{equation*}
where $\Gamma _{d}$ is the multivariate Gamma function. The expected mean is 
$\mathbf{m}_y$, the expected covariance matrix is $\frac{1}{\kappa _{y}-d-1}
\Psi _{y}$ if $\kappa_y > d + 1$, and as $\nu_y$ and $\kappa_y$ increase $%
f(\rho_y)$ becomes more ``informative.'' The probability of label function $%
\phi_S$ under this RLPP is found from~\eqref{eq:label_probability_main} and~%
\eqref{eq:label_probability} as 
\begin{align}
P(\Phi_S = \phi _{S} | S) &\propto P(\Phi = \phi) \prod_{y=1}^{l}\frac{%
\Gamma _{d}(\frac{\kappa _{y}+n_{y}}{2})}{|n_{y}+\nu _{y}|^{\frac{d}{2}%
}|\Psi _{y}+\Psi _{y}^{\ast }|^{\frac{\kappa _{y}+n_{y}}{2}}},
\label{eq:probability_model1:general}
\end{align}
where 
\begin{equation}
\Psi _{y}^{\ast } =(n_{y}-1)\widehat{\Sigma }_{y}+\frac{\nu _{y} n_{y}}{\nu
_{y} + n_{y}}\left( \widehat{\mu }_{y}-\mathbf{m}_{y}\right) \left( \widehat{%
\mu }_{y}-\mathbf{m}_{y}\right) ^{T}  \label{eq:posterior_hyperparameter}
\end{equation}
for $n_y = 2$, $\Psi _{y}^{\ast }=\frac{\nu _{y}}{\nu _{y} + 1}\left( 
\widehat{\mu }_{y}-\mathbf{m}_{y}\right) \left(\widehat{\mu }_{y}-\mathbf{m}%
_{y}\right) ^{T}$ for $n_y = 1$, and $\Psi_y^{\ast} = 0$ for $n_y = 0$. If $%
\nu_1 = \cdots = \nu_l$, $\kappa_1 = \cdots = \kappa_l$ and $P(\Phi = \phi)$
is such that the size of each cluster is fixed and partitions with clusters
of the specified sizes are equally likely, then for any $\phi_S$ inducing
clusters of the correct sizes, 
\begin{equation}
P(\Phi_S = \phi _{S} | S) \propto \prod_{y=1}^{l}|\Psi _{y}+\Psi _{y}^{\ast
}|^{-\frac{\kappa _{y}+n_{y}}{2}}.  \label{eq:probability_model2:stratified}
\end{equation}
Similar derivations for the posterior on parameters under Gaussian mixture
models can be found in~\cite{degroot2005optimal}, and on
label functions under Gaussian mixture models can be found in~\cite%
{binder1978bayesian}.

\section{Robust Clustering Operators}

Under a known RLPP $(\Xi ,\Lambda )$, optimization in the Bayes clusterer is
over the set $\bar{\mathcal{C}}$ of all clustering algorithms with respect
to the clustering error, 
\begin{equation}
\zeta ^{\ast }=\arg \min_{\zeta \in \bar{\mathcal{C}}}\varepsilon (\zeta );
\label{eq:Bayes_clusterer}
\end{equation}
however, in practice the RLPP is likely to be uncertain. In this section we
present definitions for optimal Bayesian robust clustering and show that IBR
clusterers solve an optimization problem of the same form as in~%
\eqref{eq:Bayes_clusterer} under an effective process.

\subsection{Definitions of Robust Clustering}

We present three robust clustering operators: minimax robust clustering,
model-constrained Bayesian robust (MCBR) clustering, and intrinsically
optimal Bayesian robust (IBR) clustering. Our main interest is in IBR
clustering. The first two methods are provided to emphasize parallels
between the new theory and existing robust operator theory from filtering
and classification.

Consider a parameterized uncertainty class of RLPPs $(\Xi _{\theta },\Lambda
_{\theta }),\theta \in \Theta $, where $\Xi _{\theta }$ is a point process
on $(\mathbf{N},\mathcal{N})$, $\Lambda _{\theta }=\{\Phi _{\theta ,S}:S\in 
\mathbf{N}\}$ is a random labeling on $\mathbf{N}$ consisting of a random
label function $\Phi _{\theta ,S}$ for each $S$, and $\varepsilon _{\theta
}(\zeta )$ is the error of cluster operator $\zeta $ for state $\theta $.

A \emph{minimax robust clusterer} $\zeta _{\mathrm{MM}}$ is defined by~%
\eqref{MC_minimax} with $\mathcal{C}_{\Theta }$ being the set of
state-specific Bayes clusterers and $\varepsilon _{\theta }(\zeta )$ in
place of $\gamma _{\theta }(\psi )$. An \textit{MCBR cluster operator} $%
\zeta _{\mathrm{MCBR}}$ is defined by~\eqref{MCBR} with $\varepsilon
_{\theta }(\zeta )$ in place of $\gamma _{\theta }(\psi )$. 

Our focus is on optimization over the full class $\bar{\mathcal{C}}$ of
cluster operators. This yields an \textit{IBR cluster operator}, 
\begin{equation}
\zeta _{\mathrm{IBR}}=\arg \min_{\zeta \in \bar{\mathcal{C}}}E_{\theta
}[\varepsilon _{\theta }(\zeta )].  \label{eq:IBR_definition}
\end{equation}%
In analogy to \cite{dalton2014intrinsically}, where effective
characteristics for IBR linear filtering were derived from effective random
signal processes, here we show how IBR cluster operators can be found via
effective random labeled point processes.

\subsection{Effective Random Labeled Point Processes}

We begin with two definitions.

\begin{defn}
An RLPP is \textit{solvable under clusterer class} $\mathcal{C}$ if 
\begin{align*}
\zeta^{\ast} &= \arg \min_{\zeta \in \mathcal{C}} \varepsilon(\zeta)
\end{align*}
can be solved under this process.
\end{defn}

\begin{defn}
Let $\Theta $ be an uncertainty class of RLPPs having prior $\pi (\theta )$.
An RLPP $(\Xi _{\mathrm{eff}},\Lambda _{\mathrm{eff}})$ is an \textit{%
effective RLPP under clusterer class} $\mathcal{C}$ if for all $\zeta \in 
\mathcal{C}$ both the expected clustering error $E_{\theta }\left[
\varepsilon _{\theta }(\zeta )\right] $ under the uncertainty class of RLPPs
and the clustering error $\varepsilon _{\mathrm{eff}}(\zeta )$ under $(\Xi _{%
\mathrm{eff}},\Lambda _{\mathrm{eff}})$ exist and 
\begin{equation}
E_{\theta }\left[ \varepsilon _{\theta }(\zeta )\right] =\varepsilon _{%
\mathrm{eff}}(\zeta ).  \label{eq:effective_RLPP_definition}
\end{equation}
\end{defn}

\begin{thm}
Let $\Theta $ parameterize an uncertainty class of RLPPs with prior $\pi (\theta )$.
If there exists a solvable effective RLPP $(\Xi _{\mathrm{eff}},\Lambda _{%
\mathrm{eff}})$ under clusterer class $\mathcal{C}$ with optimal clusterer $%
\zeta _{\mathrm{eff}}^{\ast }$, then $\zeta _{\mathrm{eff}}^{\ast }=\arg
\min_{\zeta \in \mathcal{C}}E_{\theta }[\varepsilon _{\theta }(\zeta )]$. If 
$\mathcal{C}=\mathcal{C}_{\Theta }$, then $\zeta _{\mathrm{MCBR}}^{\ast
}=\zeta _{\mathrm{eff}}^{\ast }$, and if $\mathcal{C}=\bar{\mathcal{C}}$,
then $\zeta _{\mathrm{IBR}}^{\ast }=\zeta _{\mathrm{eff}}^{\ast }$. \label%
{theorem:evaluating_IBR}
\end{thm}

\begin{proof}
The proof is immediate from the definition of an effective RLPP and~\eqref{eq:Bayes_clusterer}: 
\begin{equation*}
\arg \min_{\zeta \in \mathcal{C}}E_{\theta }[\varepsilon _{\theta }(\zeta
)]=\arg \min_{\zeta \in \mathcal{C}}\varepsilon _{\mathrm{eff}}(\zeta
)=\zeta _{\mathrm{eff}}^{\ast }.
\end{equation*}
The solutions for MCBR and IBR clustering follow from their
definitions. 
\end{proof}

To find an MCBR or IBR clusterer, we first seek an effective RLPP. This
effective RLPP is not required to be a member of the uncertainty class
parameterized by $\theta $, but must be solvable. If $(\Xi _{\mathrm{eff}%
},\Lambda _{\mathrm{eff}})$ is an effective RLPP under clusterer class $%
\mathcal{C}$, then it is an effective RLPP under any smaller clusterer
class. Hence, an effective RLPP found for IBR clustering is also an
effective RLPP for MCBR clustering. However, not only are IBR clusterers
better performing than MCBR clusterers, they are typically much easier to
find analytically. In particular, the IBR clusterer is directly solved by
importing methods from Bayes clustering theory, i.e., one may solve~%
\eqref{eq:Bayes_clusterer} by minimizing the partition error over all
partitions of a point set $S$ under the effective RLPP. The MCBR clusterer,
on the other hand, is significantly hampered by computational overhead in
finding $\mathcal{C}_{\Theta }$ and actually evaluating the clustering error
for each $\zeta \in C_{\Theta }$. The next theorem addresses the existence
of effective RLPPs.

\begin{thm}
Let $\Theta $ parameterize an uncertainty class $\{(\Xi _{\theta },\Lambda
_{\theta })\}_{\theta \in \Theta }$ of RLPPs with prior $\pi (\theta )$.
There exists an RLPP, $(\Xi _{\mathrm{eff}},\Lambda _{\mathrm{eff}})$, such
that 
\begin{equation}
E_{\theta }\left[ E_{\Xi _{\theta },\Lambda _{\theta }}\left[ \left. g(\Xi
_{\theta },\Phi _{\theta ,\Xi _{\theta }})\right\vert \theta \right] \right]
=E_{\Xi _{\mathrm{eff}},\Lambda _{\mathrm{eff}}}\left[ g(\Xi _{\mathrm{eff}%
},\Phi _{\mathrm{eff},\Xi _{\mathrm{eff}}})\right]
\label{eq:effective_RLPP_error}
\end{equation}
for any real-valued measurable function, $g$. \label{theorem:effective_RLPP}
\end{thm}

\begin{proof}
Suppose that the parameter $\theta $ is a realization of a random vector, $\vartheta :(\Omega ,\mathcal{A},P)\rightarrow (\Theta ,\mathcal{B})$. Then $\{\vartheta ^{-1}(\theta ):\theta \in \Theta \}$ partitions the sample
space, $\Omega $. The point process $\Xi _{\theta }$ is thus a mapping
\begin{equation*}
\Xi _{\theta }:(\vartheta ^{-1}(\theta ),\mathcal{A}\cap \vartheta
^{-1}(\theta ),P_{\theta })\rightarrow (\mathbf{N},\mathcal{N}),
\end{equation*}
where $P_{\theta }$ is the conditional probability and we assume $\nu
_{\theta }(Y)=P_{\theta }(\Xi _{\theta }^{-1}(Y))$ for all $Y\in \mathcal{N}$
is known. Write the random labeling as $\Lambda _{\theta }=\{\Phi _{\theta
,S}:S\in \mathbf{N}\}$, where $\Phi _{\theta ,S}$ has a probability mass
function $P(\Phi _{\theta ,S}=\phi _{S}|\theta ,S)$ on $L^{S}$. Given any
real-valued measurable function $g$ mapping from point set and label
function pairs, let $X=g(\Xi ,\Phi _{\Xi })$ be a random variable where $(\Xi ,\Phi _{\Xi })$ is drawn from $\{(\Xi _{\theta },\Lambda _{\theta
})\}_{\theta \in \Theta }$ with prior $\pi (\theta )$, and note $E_{\theta
}[E[X|\theta ]]=E[X]$.

Let $\Xi _{\mathrm{eff}}:(\Omega ,\mathcal{A},P)\rightarrow (\mathbf{N},\mathcal{N})$ be a mapping, 
where given a fixed $\omega \in \Omega$ we have a corresponding fixed realization $\theta = \vartheta(\omega)$ and we define
$\Xi _{\mathrm{eff}}(\omega)=\Xi _{\theta }(\omega )$. Note
that
\begin{equation*}
\nu (Y)\equiv P(\Xi _{\mathrm{eff}}^{-1}(Y))=E_{\theta }[\nu _{\theta }(Y)]
\end{equation*}
and $\Xi _{\mathrm{eff}}$ is a random point process. Define $\Lambda _{\mathrm{eff}}=\{\Phi _{\mathrm{eff},S}:S\in \mathbf{N}\}$, where $\Phi _{\mathrm{eff},S}$ has a probability mass function
\begin{equation*}
P(\Phi _{\mathrm{eff},S}=\phi _{S}|S)=E_{\theta }[P(\Phi _{\theta ,S}=\phi
_{S}|\theta ,S)]
\end{equation*}
for all $\phi _{S}\in L^{S}$. Thus, $\Lambda _{\mathrm{eff}}$ is a random
labeling. Let $Z=g(\Xi _{\mathrm{eff}},\Phi _{\mathrm{eff},\Xi _{\mathrm{eff}}})$ be a random variable where $(\Xi _{\mathrm{eff}},\Phi _{\mathrm{eff},\Xi _{\mathrm{eff}}})$ is drawn from the RLPP we have constructed, $(\Xi _{\mathrm{eff}},\Lambda _{\mathrm{eff}})$, and note $E[X]=E[Z]$. 
\end{proof}

Theorem~\ref{theorem:effective_RLPP} applies for any function $g(S,\phi
_{S}) $, including the cluster mismatch error $g(S,\phi _{S})=\varepsilon
(S,\phi _{S},\zeta (S))$, for any clusterer $\zeta \in \bar{\mathcal{C}}$.
Thus, \eqref{eq:effective_RLPP_error} implies 
\begin{align*}
E_{\theta }\left[ \varepsilon _{\theta }(\zeta )\right] =& E_{\theta }\left[
E_{\Xi _{\theta },\Lambda _{\theta }}\left[ \varepsilon (\Xi _{\theta },\Phi
_{\theta ,\Xi _{\theta }},\zeta (\Xi _{\theta }))|\theta \right] \right] \\
=& E_{\Xi _{\mathrm{eff}},\Lambda _{\mathrm{eff}}}\left[ \varepsilon (\Xi _{%
\mathrm{eff}},\Phi _{\mathrm{eff},\Xi _{\mathrm{eff}}},\zeta (\Xi _{\mathrm{%
eff}}))\right] =\varepsilon _{\mathrm{eff}}(\zeta ).
\end{align*}%
Hence, $(\Xi _{\mathrm{eff}},\Lambda _{\mathrm{eff}})$ is an effective RLPP
on $\bar{\mathcal{C}}$, covering MCBR and IBR clusterers.

The following corollary shows that for separable RLPPs, the effective RLPP
is also separable and aggregates uncertainty within and between models.

\begin{cor}
Let each RLPP in the uncertainty class be parameterized by $\rho $ with
prior density $f(\rho |\theta )$, let $\Phi$ be an independent labeling
process with a probability mass $P(\Phi =\phi )$ that depends on neither $%
\theta $ nor $\rho $, and denote the conditional distribution of points by $%
f(\mathbf{x}|y,\rho ,\theta )$. Then the effective RLPP is separable with
parameter $[\theta ,\rho ] $, prior $f(\theta ,\rho )$, an independent
labeling process with probability mass $P(\Phi =\phi )$, and conditional
distributions $f(\mathbf{x}|y,\rho ,\theta )$. \label%
{corollary:effective_RLPP}
\end{cor}

\begin{proof}
Let the number of points, $n$, and the label function $\phi : \{1, \ldots, n\} \to L^n$ be fixed. For
a fixed $\theta $, the effective random point process $\Xi _{\mathrm{eff}}(\omega )$ is set equal to $\Xi _{\theta }(\omega )$. 
Equivalently, a realization of $S=\{\mathbf{x}_{1},\ldots ,\mathbf{x}_{n}\}$
under the effective RLPP is governed by the distribution 
\begin{equation*}
f(S | \phi)
= \int \left( \int 
\Bigg( \prod_{i=1}^{n}f(\mathbf{x}_{i}|\phi (i), \rho ,\theta ) \Bigg) f(\rho |\theta ) d\rho \right) \pi (\theta ) d\theta.
\end{equation*}
This is equivalent to a separable random point process with parameter $[\theta
,\rho ]$, prior $f(\theta ,\rho )=\pi(\theta )f(\rho |\theta )$ and
conditional distributions $f(\mathbf{x}|y,\rho ,\theta )$. Since the
labeling process is independent, the full effective RLPP is the separable RLPP
given in the statement of the corollary.
\end{proof}

A graphical model of the uncertainty class of RLPPs assumed in Corollary~\ref{corollary:effective_RLPP} is provided in Fig.~\ref{fig:graphical_model}.
The IBR clusterer can be found as follows. 
\begin{enumerate}
	\item We input an uncertainty class of RLPPs of the form stated in Corollary~\ref{corollary:effective_RLPP} and illustrated in Fig.~\ref{fig:graphical_model}.  In particular, we require the sample size, $n$, prior $\pi(\theta)$, label process probability mass function $P(\Phi =\phi)$, parameter prior $f(\rho |\theta )$ and conditional density $f(\mathbf{x}|y,\rho ,\theta )$.  
	
	\item By Corollary~\ref{corollary:effective_RLPP}, the effective RLPP is found by merging uncertainty in the state (across RLPPs) and parameters (within RLPPs). In particular, the effective RLPP is characterized by the sample size $n$, label process probability mass function $P(\Phi =\phi)$, parameter prior $f(\theta, \rho)$ and density $f(\mathbf{x}|y,\rho ,\theta )$.  
	
	\item By Theorem~\ref{theorem:evaluating_IBR}, the IBR clusterer is the Bayes (optimal) clusterer under the effective RLPP.  Given point set $S$, the IBR clusterer outputs
	the partition $\mathcal{P}_S$ corresponding to the minimal
	error $\varepsilon(S, \mathcal{P}_S)$ in~%
	\eqref{eq:error_of_partition:vectors}.  The natural cost function $c_S$ is a constant function given by~\eqref{eq:cost_function:final}, the partition probabilities are given by~\eqref{eq:partition_probability}, and the label function probabilities $P(\Phi_S =
	\phi_S | S)$ under the effective (separable) RLPP are given by~\eqref{eq:label_probability_main} and the likelihood function~\eqref{eq:label_probability_joint} with $f(\theta, \rho)$ in place of $f(\rho)$ and $f(\mathbf{x}|y,\rho ,\theta )$ in place of $f(\mathbf{x}|y,\rho )$.  
\end{enumerate}
In practice, the primary issues are: 
(a) deriving an analytical form for the label function probability, $P(\Phi_S =
\phi_S | S)$, 
(b) evaluating the natural cost, $c_S$, for all pairs of partitions, 
and (c) evaluating partition errors, $\varepsilon(S, \mathcal{P}_S)$, for all partitions.  
Note that $P(\Phi_S = \phi_S | S)$ is available for Gaussian separable RLPPs in~\eqref{eq:probability_model1:general}.  
Issues (b) and (c) may also be alleviated using optimal and suboptimal algorithms, as discussed in~\cite{dalton2015analytic}.  

\begin{figure}[tb!]
	\centering
	\includegraphics[width=0.35\textwidth]{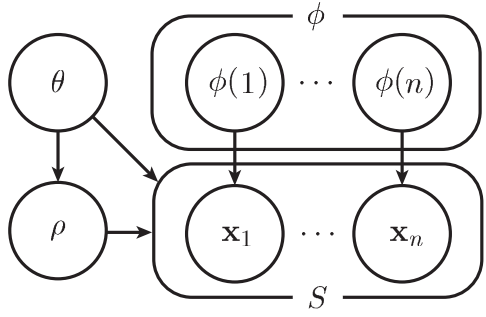}
	\caption{A graphical model of the uncertainty class of RLPPs assumed in Corollary~\ref{corollary:effective_RLPP}.  The parameter $\theta$ is governed by a prior distribution $\pi(\theta)$ and indexes each RLPP in the uncertainty class.  The number of points, $n$, may be generated from an independent process, or considered fixed.  For fixed $n$, the label function, $\phi$, is generated according to the probability mass function $P(\Phi =\phi)$.  Given $\theta$, $\rho$ is generated from the density $f(\rho |\theta )$, and each point in the point set $S = \{\mathbf{x}_1, \ldots, \mathbf{x}_n\}$ is drawn from the density $f(\mathbf{x}|y,\rho ,\theta )$, where the corresponding label for point $\mathbf{x}_i$ is $y = \phi(i)$. }
	\label{fig:graphical_model}
\end{figure}

\section{Robust Clustering Under Gaussian RLPPs}

\label{sec:simulation_Gaussian}

Consider synthetic Gaussian data with $l=2$ clusters in $d=1$, $2$, $10$, $%
100$ and $1,000$ dimensions. The state of nature is composed of the cluster
covariances, and for a given state of nature the point process generates
equal sized Gaussian clusters with random means and the corresponding
covariances. Formally, we parameterize the uncertainty class of RLPPs with $%
\theta =[\theta _{1},\theta _{2}]$, where $\theta _{y}=\Sigma _{y}$, and
each $\Sigma _{y}$ is drawn independently from an inverse-Wishart
distribution with $\kappa _{y}$ degrees of freedom and scale matrix $\Psi
_{y}$. The RLPP in the uncertainty class corresponding to $\theta $, $(\Xi
_{\theta },\Lambda _{\theta })$, is a separable RLPP with parameter $\rho
_{y}=\mu _{y}$, Gaussian prior $f(\rho _{y})$ with mean $\mathbf{m}_{y}$ and
covariance $\Sigma _{y}/\nu _{y}$, and Gaussian conditional distributions $f(%
\mathbf{x}|y,\rho _{y},\theta _{y})$ with mean $\mu _{y}$ and covariance $%
\Sigma _{y}$. We set $\kappa _{1}=\kappa _{2}=d+2$, $\Psi _{1}=\Psi _{2}$ to
be $d\times d$ identity matrices, $\nu _{1}=\nu _{2}=1$, and $\mathbf{m}_{1}=%
\mathbf{m}_{2}$ to be all-zero vectors. The number of points, $n=n_{1}+n_{2}$%
, is set to $10$ or $100$ with $n_{1}=n_{2}$, and the labels are permuted.
Thus, the true distribution on label functions, $P(\Phi =\phi )$, has a
support on the set of label functions that assign the correct number of
points to each cluster, and is uniform on its support. %

For each combination of $d$ and $n$, we generate $1,000$ states of nature, $%
\theta$, and one point set per state of nature from the corresponding
separable RLPP $(\Xi_{\theta}, \Lambda_{\theta})$. For each point set, we
run several classical clustering algorithms: fuzzy $c$-means (FCM), $k$%
-means (KM), hierarchical clustering with single linkage (H-S), hierarchical
clustering with average linkage (H-A), hierarchical clustering with complete
linkage (H-C), and a clusterer that produces a random partition with equal
sized clusters for reference (Random). More details about these algorithms
may be found in~\cite{dalton2009clustering}. In addition, we cluster using
expectation maximization for Gaussian mixture models (EM), and a method that
minimizes a lower bound on the posterior expected variation of information
under an estimated posterior similarity matrix generated from samples of a
Gibbs sampler for Gaussian mixture models (MCMC)~\cite{wade2015bayesian}. EM
is run using the mclust package in R with default settings~\cite%
{fraley2002model, fraley2012normal, manualR}. The Gibbs sampler is
implemented using the bayesm package in R with $18,000$ samples generated
after a burn-in period of $2,000$ samples, and otherwise default settings~%
\cite{Rossi2017bayesm}. The posterior similarity matrix is estimated using
the mcclust package in R~\cite{Fritsch2012mcclust}, and minimization with
respect to variation of information is implemented with the mcclust.ext
package in R~\cite{Wade2015mcclust}. We also implement EM informed with the
``correct'' hyperparameters, $\kappa_y$, $\Psi_y$, $\nu_y$ and $\mathbf{m}_y$
(EM-I) and MCMC informed with the ``correct'' hyperparameters (MCMC-I).

To find the IBR clusterer, the effective RLPP, $(\Xi, \Lambda)$, is
constructed using Corollary~\ref{corollary:effective_RLPP}, which states
that the effective RLPP merges uncertainty in the state $\theta$ with
uncertainty in the parameter $\rho$. In this case, the effective RLPP is
precisely the separable RLPP presented in Section~\ref{sec:Gaussian_RLPPs},
which accounts for both random means in $\rho$ and random covariances in $%
\theta$. The effective RLPP is solvable (at least for small point sets)
using the Bayes clusterer presented in~\cite{dalton2015analytic}. By Theorem~%
\ref{theorem:evaluating_IBR}, the IBR clusterer is equivalent to the Bayes
clusterer under the effective RLPP. Thus, the IBR clusterer can be found
when $n = 10$ by evaluating~\eqref{eq:error_of_partition:final} for all
partitions using~\eqref{eq:partition_probability} and~%
\eqref{eq:probability_model2:stratified}, and choosing the minimizing
partition. When $n = 100$, we approximate the IBR clusterer (IBR-A) using a
sub-optimal algorithm, Suboptimal Pseed Fast, presented in~\cite%
{dalton2015analytic}, which finds the maximum probability partition for a
random subset of $10$ points, generalizes these clusters to the full point
set using a QDA classifier, iteratively searches for the highest probability
partition on the full point set by considering all partitions with at most
two points clustered differently from the best partition found so far, and
finally chooses the highest probability partition resulting from $10$
repetitions with different initial subsets of points. MCBR and minimax
robust clusterers are not found, since they are computationally infeasible.
Furthermore, having found an IBR clusterer one would certainly not use an
MCBR clusterer and very likely not use a minimax robust clusterer.

For each point set and each algorithm, we find the cluster mismatch error
between the true partition and the algorithm output using~%
\eqref{eq:mismatch_error}. For each combination of $d$ and $n$ and each
algorithm, we approximate the average partition error, $E_{\theta
}[\varepsilon _{\theta }(\zeta )]$, under the natural cost function using
the average cluster mismatch error across all $1,000$ point sets. Figure~\ref%
{fig:clustering_error_Gaussian}(a) presents a graph of these errors with
respect to $d$ for $n=10$, and similarly Figure~\ref%
{fig:clustering_error_Gaussian}(b) presents performance for $n=100$. These
graphs support the fact that the IBR clusterer is optimal in these
simulations when $n=10$, and that the approximate IBR clusterer is close to
optimal when $n=100$. Indeed, the IBR clusterer performs significantly
better than all other algorithms under high dimensions.

When the number of points is large ($n = 100$) and the number of dimensions
is smaller than the number of points, the performances of EM and EM-I are
very close to the approximate IBR clusterer. However, when the number of
points is small, or the number of dimensions is larger than the number of
points, these algorithms tend to be similar to FCM and KM. This is most
likely because mclust tests several different modeling assumptions regarding
the covariances, and uses the Bayesian information criterion (BIC) to select
a final output partition. When $n$ is small relative to $d$, the full
covariances of the Gaussian mixtures cannot be estimated well, so there is a
tendency to select simpler models that assume covariances are equal and
circular, which is essentially the same assumption made by FCM and KM.
Finally note that MCMC by default uses a particular normal-inverse-Wishart
prior with hyperparameters that do not match the ``correct''
hyperparameters. The fact that MCMC-I performs much better than MCMC
suggests that this method may be quite sensitive to the priors, especially
when the sample size is small. 

\begin{figure*}[tb!]
\centering
\subfloat[]{\includegraphics[width=0.48\textwidth]{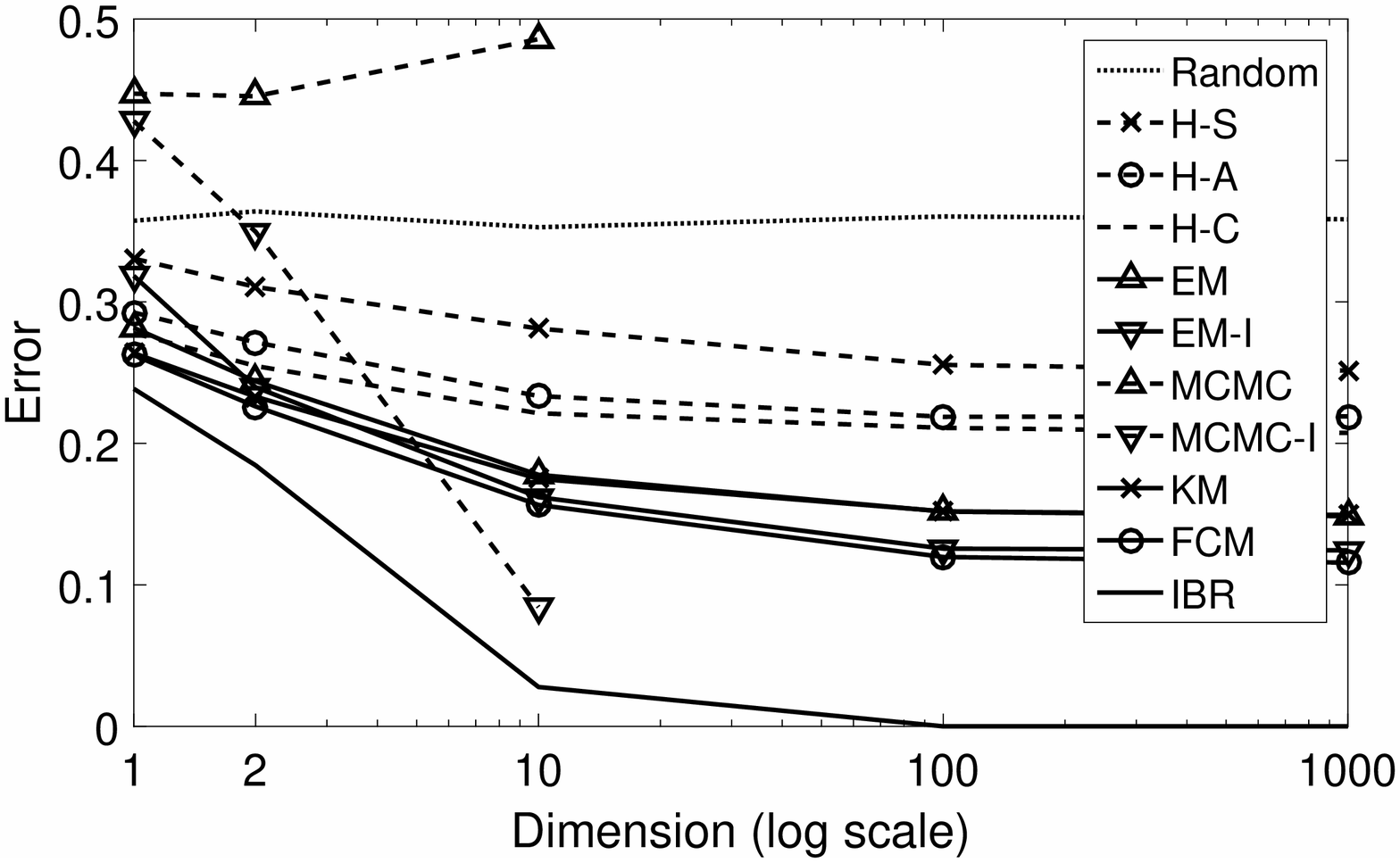}} 
\hfill
\subfloat[]{\includegraphics[width=0.48\textwidth]{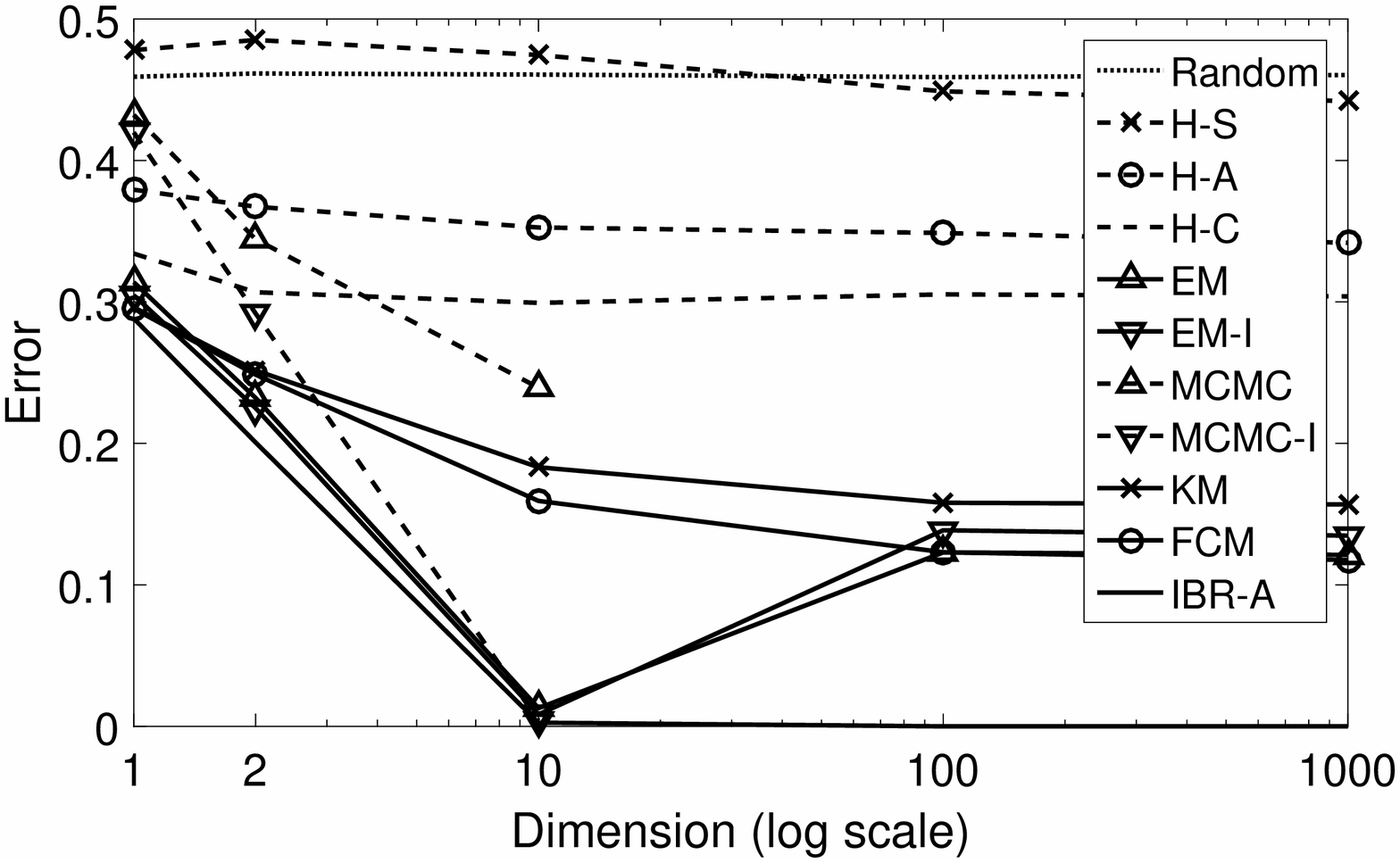}}
\caption{Average cluster mismatch error for Gaussian RLPPs: (a) $n = 10$,
(b) $n = 100$.}
\label{fig:clustering_error_Gaussian}
\end{figure*}

\section{Robust Clustering in Granular Imaging}

While digital photography may now dominate over chemical photography,
silver-based imaging remains important and is currently growing in use.
Research remains active. Crystal shape is of particular importance. For many
years granulometric analysis has been important in particle and texture
analysis. In particular, morphological granulometries can generate image
features relating to the size, shape, and concentration of particles. We
present an application of robust clustering for images of silver-halide
photographic T-grain crystals with respect to grain proportions using
granulometric features.


\subsection{Morphological Granulometries}

A basic model for silver-halide emulsions includes grains that are
equilateral triangles, hexagons formed by removing triangle corners, rods
(rectangles), and ill-formed blobs. To simplify calculations, we focus on a
binary image model using only triangles and rods. In film grade emulsions
grains overlap, but for laboratory analysis diluted emulsions with
negligible overlapping can be produced, thus we also focus on images with
non-overlapping grains.

Morphological granulometries are particularly well-suited for modeling and
processing binary images consisting of grains of different sizes and shapes.
The most commonly employed granulometry is a family of parameterized
morphological openings: for a convex, compact \emph{structuring element}
(set) $B$, a \textit{granulometry} $\{\Psi _{t}\}$ is defined by $\Psi
_{t}(I)=I\circ tB$ for $t>0$ and $\Psi _{0}(I)=I$, where $I\circ tB=\cup
\{tB+x:$ $tB+x\subset I\}$ is the opening of image (set) $I$ by $tB$ (more
general granulometries exist~\cite{matheron1975random}). If $%
\Omega _{I}(t)$ is the area of $\Psi _{t}(I)$, then $\Omega _{I}(t)$ is a
decreasing function of $t$, known as a \textit{size distribution}. A
normalized size distribution is defined by $\Phi _{I}(t)=1-\Omega
_{I}(t)/\Omega _{I}(0)$. If $I$ is compact and $B$ consists of more than a
single point, then $\Phi _{I}(t)$ increases from 0 to 1 and is continuous
from the left. Thus, it defines a probability distribution function called
the \textit{pattern spectrum} of $I$ (relative to $B$). Moments of $\Phi
_{I}(t)$ are used for image classification and segmentation~\cite%
{Dougherty:91}. $\Phi _{I}(t)$ is a random function
and its non-central moments (called \textit{granulometric moments}) are
random variables.

In this work, we use granulometric moments as features for clustering. Given
a set $I$, we extract as features the first $q$ granulometric moments of $I$
generated by granulometries arising from $p$ structuring elements $%
B_{1},B_{2},\ldots ,B_{p}$, where we denote the $k$th granulometric moment
corresponding to $B_{j}$ by $\mu ^{(k)}(I,B_{j})$ for $j=1,2,\ldots ,p$ and $%
k=1,2,\ldots ,q$. Consider a random set $I$ of the form 
\begin{equation}
I=\bigcup_{i=1}^{m}\bigcup_{j=1}^{N_{i}}\left( r_{ij}A_{i}+x_{ij}\right) ,
\label{eq:image_model}
\end{equation}%
where $A_{1},A_{2},\ldots ,A_{m}$ are compact sets called \emph{primitives}, 
$r_{ij}$ and $x_{ij}$ specify the radius (grain size) and center of the $j$%
th grain of primitive type $i$, respectively, and all $N=N_{1}+\ldots +N_{m}$
grains are mutually disjoint.

In the silver halide application, we assume preprocessed images are well
modeled by~\eqref{eq:image_model}, where $m=2$, $A_{1}$ is an equilateral
triangle with horizontal base, and $A_{2}$ is a vertical rod with height $5$
times its base. Without loss of generality, we assume both primitives have
unit area, i.e., $\nu \lbrack A_{1}]=\nu \lbrack A_{2}]=1$, and we denote
the grain proportions by $b_{1}$ and $b_{2}$. We further assume the $r_{ij}$
are independent with the $r_{i1},\ldots ,r_{iN_{i}}$ identically
distributed, where the \emph{grain sizing distribution} for primitive $i$
has the property $E[r_{ij}^{k}]=\gamma _{ik}\beta ^{k}$ for all $k>0$ and 
$\gamma _{ik}$ and $\beta $ are positive constants. If $r_{ij}\sim $ gamma$%
(\alpha _{i},\beta )$, $\beta $ being the scale parameter for both
primitives, then this property holds with $\gamma _{ik}=\Gamma (\alpha
_{i}+k)/\Gamma (\alpha _{i})$.

For the morphological opening, we use $p=2$ structuring elements, where $%
B_{1}$ and $B_{2}$ are, respectively, vertical and horizontal linear
structuring elements. The first $q=2$ granulometric moments for $B_{1}$ and $%
B_{2}$ are 
\begin{equation*}
\mathbf{z}=%
\begin{bmatrix}
\mu ^{(1)}(I,B_{1})\;\;\mu ^{(1)}(I,B_{2})\;\;\mu ^{(2)}(I,B_{1})\;\;\mu
^{(2)}(I,B_{2})%
\end{bmatrix}%
^{T}.
\end{equation*}%
Given the constants $\mu ^{(k)}(A_{i},B_{j})$ and the radii $r_{ij}$ of all
grains, the exact moments in $\mathbf{z}$ under the granulometric model may
be found analytically (see Theorem~\ref{thm:Sand98} in the Appendix). 
In particular, $\mathbf{z}=M\mathbf{x}$, where 
\begin{equation*}
M=%
\begin{bmatrix}
\mu ^{(1)}(A_{1},B_{1}) & \mu ^{(1)}(A_{2},B_{1}) & 0 & 0 \\ 
\mu ^{(1)}(A_{1},B_{2}) & \mu ^{(1)}(A_{2},B_{2}) & 0 & 0 \\ 
0 & 0 & \mu ^{(2)}(A_{1},B_{1}) & \mu ^{(2)}(A_{2},B_{1}) \\ 
0 & 0 & \mu ^{(2)}(A_{1},B_{2}) & \mu ^{(2)}(A_{2},B_{2})%
\end{bmatrix}%
,
\end{equation*}%
and $\mathbf{x}=[x_{11},x_{21},x_{12},x_{22}]^{T}$, where 
\begin{equation}
x_{ik}=\frac{\sum_{j=1}^{N_{i}}r_{ij}^{k+2}}{\sum_{j=1}^{N_{1}}r_{1j}^{2}+%
\sum_{j=1}^{N_{2}}r_{2j}^{2}}.  \label{eq:x}
\end{equation}%
In general, the constants $\mu ^{(k)}(A_{i},B_{j})$ under convex grains can
be found using theory from~\cite{Dougherty:95}. It can be shown that for
triangle $A_{1}$ and vertical structuring element $B_{1}$ that $\mu
^{(1)}(A_{1},B_{1})=2\cdot 3^{-3/4}$ and $\mu
^{(2)}(A_{1},B_{1})=2^{-1}3^{1/2}$. Similarly, for other combinations of
primitives and structuring elements, $\mu ^{(1)}(A_{1},B_{2})=4\cdot
3^{-5/4} $, $\mu ^{(2)}(A_{1},B_{2})=2\cdot 3^{-1/2}$, $\mu
^{(1)}(A_{2},B_{1})=5^{1/2}$, $\mu ^{(2)}(A_{2},B_{1})=5$, $\mu
^{(1)}(A_{2},B_{2})=5^{-1/2}$ and $\mu ^{(2)}(A_{2},B_{2})=5^{-1}$.

In the current application, we cluster on the features $\mathbf{x}=M^{-1}%
\mathbf{z}$. 
%
Lemma~\ref{lem:granulometry} in the Appendix guarantees asymptotic joint
normality and provides analytic expressions for the asymptotic mean and
covariance of granulometric moments under multiple primitives and multiple
structuring elements. In particular, given the grain proportions $b_{1}$ and 
$b_{2}$, and the grain sizing parameters $\beta $ and $\gamma _{ik}$ for $%
i=1,2$ and $k=2,3,4$, $\mathbf{x}$ has asymptotic mean 
\begin{equation}
\frac{1}{b_{1}\gamma _{12}+b_{2}\gamma _{22}} 
\begin{bmatrix}
b_{1}\gamma _{13}\beta & b_{2}\gamma _{23}\beta & b_{1}\gamma _{14}\beta^{2}
& b_{2}\gamma _{24}\beta ^{2}%
\end{bmatrix}%
^{T}  \label{eq:mu}
\end{equation}
and covariance matrix 
\begin{equation}
\frac{1}{N(b_{1}\gamma _{12}+b_{2}\gamma _{22})^{4}} 
\begin{bmatrix}
A_{11}\beta ^{2} & A_{12}\beta ^{3} \\ 
A_{21}\beta ^{3} & A_{22}\beta ^{4}%
\end{bmatrix}%
,  \label{eq:Sigma}
\end{equation}
where the $A_{ij}$ are $2\times 2$ matrices that depend on only the $b_{i}$
and $\gamma _{ik}$, and are provided in the Appendix.

\subsection{Robust Clustering}

Suppose we are given a collection of $n$ binary images of mixtures of
silver-halide photographic T-grain crystals, where each image belongs to one
of two groups, indexed by $y=1,2$. Images in class $1$ and $2$ have
different proportions of triangles, $b_{1}$, and different sizing
parameters, thereby providing different photographic properties. Our
objective is to cluster the images into the two groups (our concern is
partitioning, not labeling) based on feature vectors $\mathbf{x}=M^{-1}%
\mathbf{z}$ obtained from moments of morphological openings $\mathbf{z}$.

Given the grain sizing distributions and a prior $f(\rho)$, the asymptotic
joint normality of $\mathbf{x}$ motivates a separable RLPP model where,
given $y$ and $\rho$, $f(\mathbf{x} | y, \rho)$ is a Gaussian distribution
with mean and covariance given by~\eqref{eq:mu} and~\eqref{eq:Sigma},
respectively. We substitute $\rho$ and $1 - \rho$ in place of $b_1$ and $b_2$
under class 1, and vice-versa under class 2. For simplicity, we assume $%
P(\Phi = \phi)$ is uniform with support such that the number of images in
each class is known.

The grain sizing distribution in a binarized image typically depends on the
image thresholding method and other factors, and thus is unknown. To account
for this, we model an uncertainty class of RLPPs parameterized by $\theta $,
where the grain sizes are assumed to be gamma$(\alpha _{iy},\beta _{y})$
distributed, the $\alpha _{iy}$ parameters are fixed and known, the $\beta
_{y}$ depend deterministically on $\theta $, and we assume $\theta $ and $%
\rho $ are mutually independent with known prior $\pi (\theta )$. From~%
\eqref{eq:label_probability_main} and~\eqref{eq:label_probability_joint},
the IBR clusterer reduces to finding the following label function
probabilities under the effective RLPP: 
\begin{align}
& P(\Phi _{S}=\phi _{S}|S)\propto P(\Phi =\phi )\times  \notag \\
& \quad \int_{0}^{\infty }\int_{0}^{1}f(S_{1}|1,\rho ,\theta )f(S_{2}|2,\rho
,\theta )f(\rho )\pi (\theta )d\rho d\theta ,
\label{eq:image_label_function_integral}
\end{align}
where 
\begin{equation*}
f(S_{y}|y,\rho ,\theta )=\prod_{\mathbf{x}\in S_{y}}f(\mathbf{x}|y,\rho
,\theta ).
\end{equation*}
Since we assume a Gaussian model, $f(S_{y}|y,\rho ,\theta )$ is precisely
the likelihood function in~\eqref{eq:gaussian_model0}. To make~%
\eqref{eq:image_label_function_integral} tractable, we assume discrete
priors on $\rho $ and $\theta $ so that the integrals can be written as sums.

\subsection{Experimental Results}

The image generation model is based on the parameterized RLPP model
described above. For a given set of images under a given RLPP with parameter 
$\theta $, which determines the sizing distribution, we generate $%
n=n_{1}+n_{2}$ binary images, where $n_{1}$ and $n_{2}$ denote the fixed
number of images from class $1$ and class $2$, respectively. Each image
contains $1,000$ non-overlapping and vertically aligned grains (triangles
and rods), and is $550\times 550$ pixels. The prior $f(\rho )$ on the
proportion of triangles for class $1$ is uniform over $500$ values from $%
0.45 $ to $0.55$, and we assume the proportion of triangles for class $2$ is 
$1-\rho$. Figure~\ref{fig:image_realization_aligned} shows three example
realizations of images with gamma$(\alpha =1.95,\beta =2)$ sizing
distributions for the triangles and gamma$(\alpha =1.97,\beta =2)$ for the
rods. Parts (a), (b), and (c) contain triangle proportions $0.45$, $0.5$,
and $0.55$, respectively.

\begin{figure*}[tb!]
\centering
\subfloat[]{\includegraphics[width=0.3\textwidth]{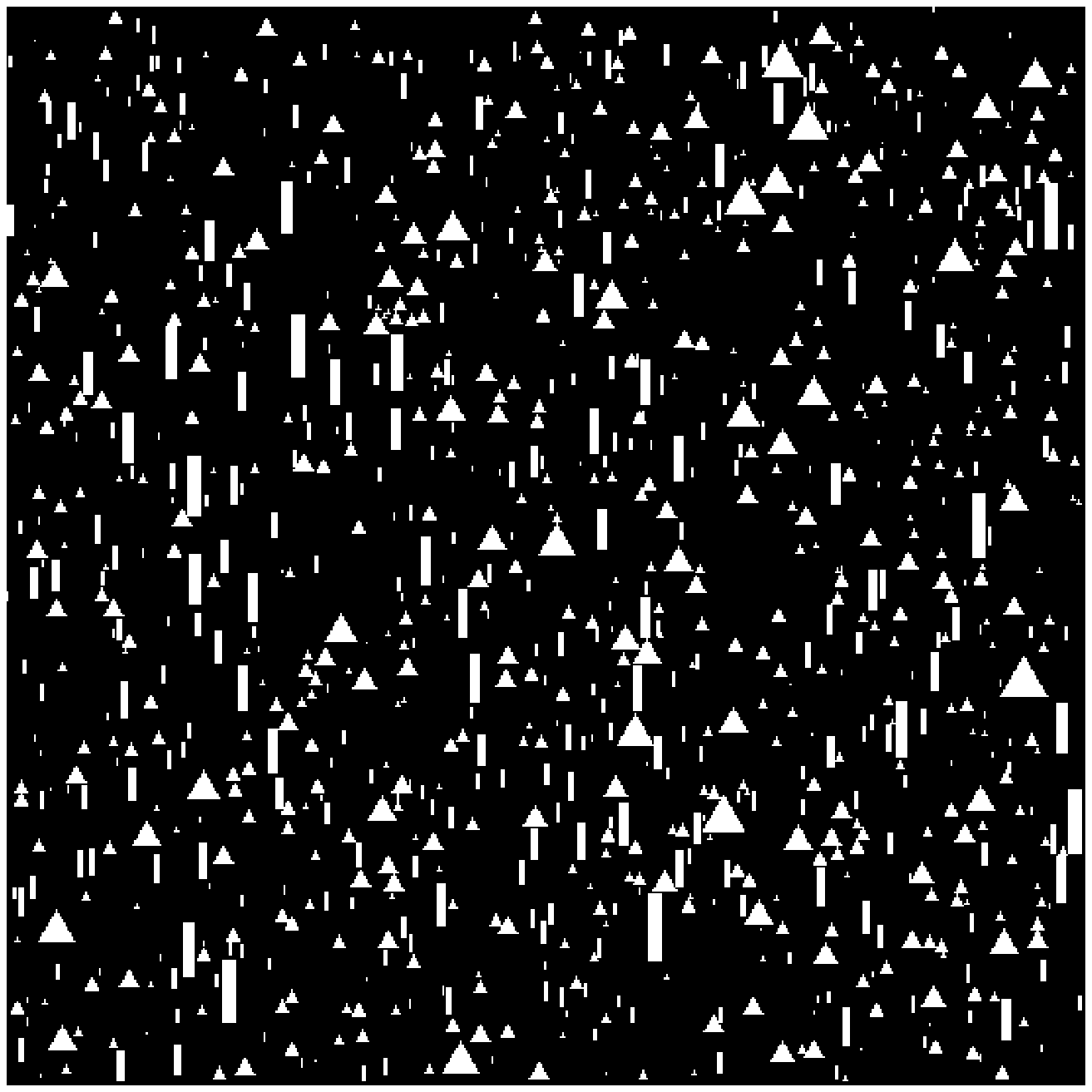}} \hfill %
\subfloat[]{\includegraphics[width=0.3\textwidth]{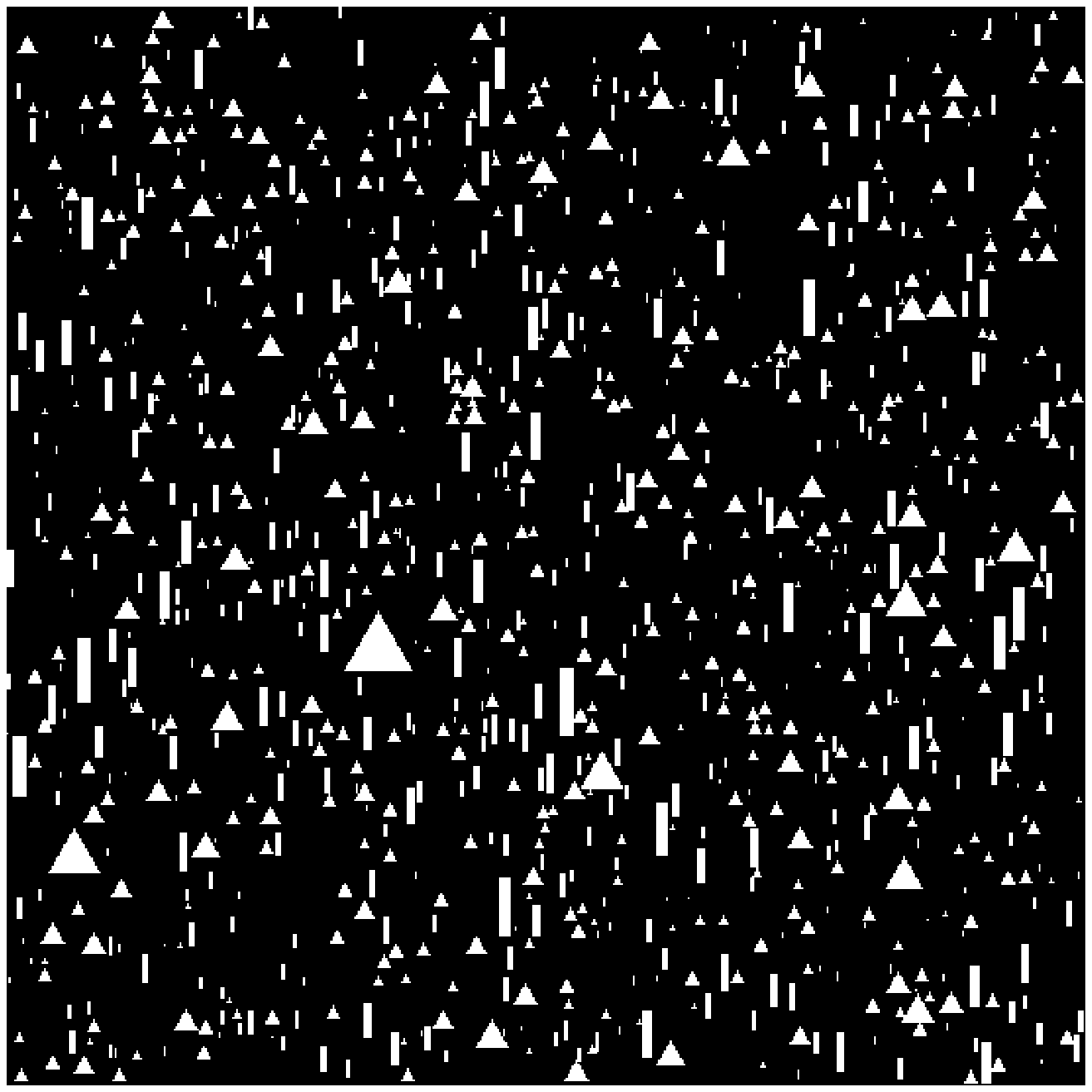}} \hfill %
\subfloat[]{\includegraphics[width=0.3\textwidth]{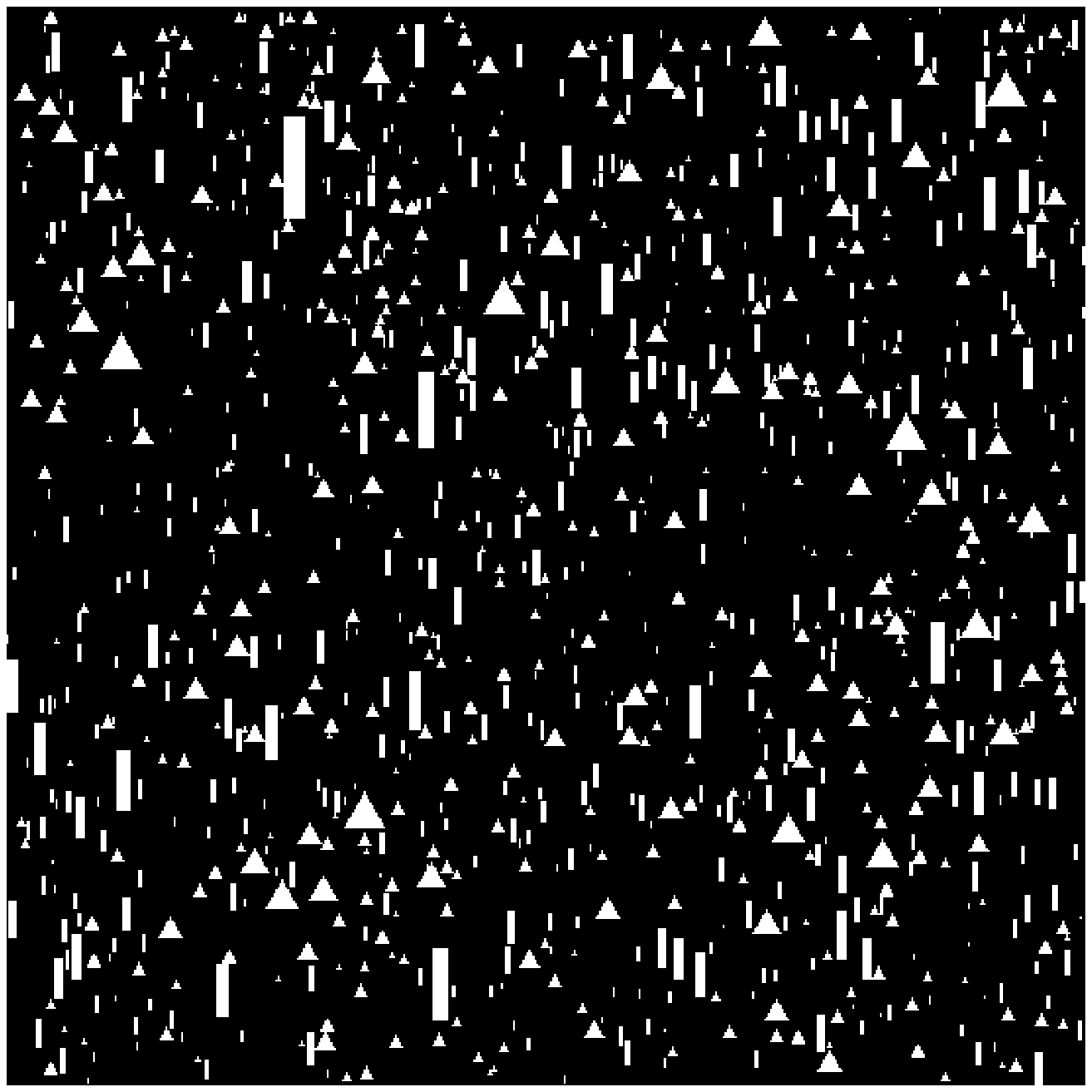}}
\caption{Examples of image realizations generated by the T-grain crystal
model. Each image contains $1,000$ grains. The sizing distribution of the
grains are gamma$(\protect\alpha = 1.95, \protect\beta = 2)$ for the
triangles and gamma$(\protect\alpha = 1.97, \protect\beta = 2)$ for the
rods. The size of each image is $550 \times 550$ pixels: (a) proportions of $%
0.45$ triangles and $0.55$ rods, (b) $0.5$ triangles and $0.5 $ rods, and
(c) $0.55$ triangles and $0.45$ rods. }
\label{fig:image_realization_aligned}
\end{figure*}

The prior $\pi (\theta )$ is uniform over $10$ values from $1.75$ to $2$. We
assume gamma$(\alpha _{iy},\beta _{y})$ sizing distributions for primitive $%
i $ under class $y$, where $\beta _{1}=\theta $, $\beta _{2}=3.75-\theta $.
For triangles, $\alpha _{1y}=1.95$ and $1.97$ for class $1$ and class $2$,
respectively, and for rods, $\alpha _{2y}=1.97$ and $1.95$ for class $1$ and
class $2$, respectively. We generate $500$ sets of images for each state,
for a total of $5,000$ sets of images. For each image, openings are found,
followed by granulometric moments $\mathbf{z}$ from the openings, and
finally a feature vector $\mathbf{x}=M^{-1}\mathbf{z}$. Figure~\ref%
{fig:scatter} provides example scatter plots of all pairs of features
extracted from $100$ images. These images correspond to $\theta =1.75$ and
the $10$ smallest values of $\rho $ (between $0.45$ and $0.452$), with $5$
images selected from each value of $\rho $ and each group.

\begin{figure*}[tb!]
\centering
\includegraphics[width=\textwidth]{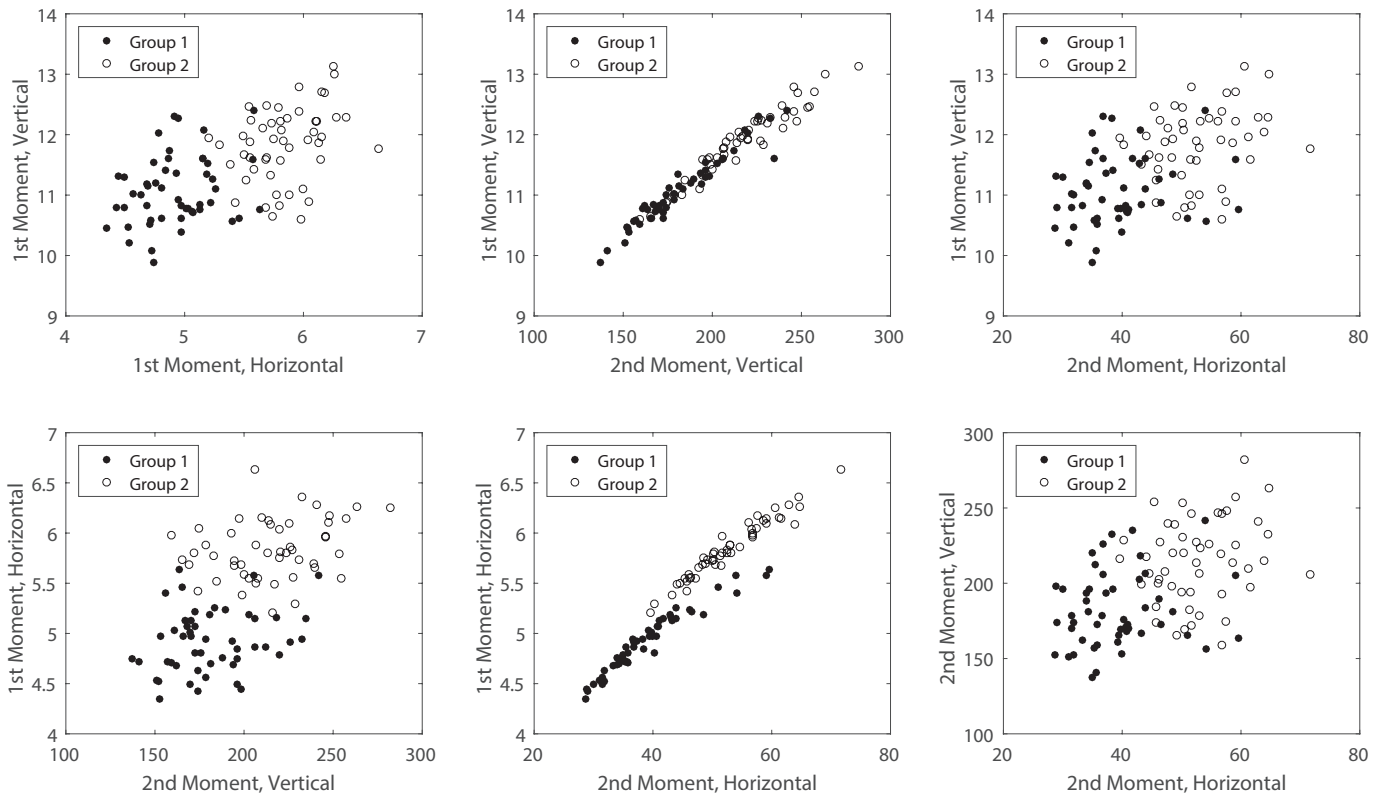}
\caption{Scatter plots of all pairs of features extracted from $100$ images.}
\label{fig:scatter}
\end{figure*}

For each set of images, 
we run FCM, KM, H-S, H-A, H-C, EM, MCMC and Random. Note EM-I and MCMC-I,
which use normal-inverse-Wishart priors on the mean and covariance pairs,
are not sensible to run here since the model uncertainty on $b_{1}$, $b_{2}$
and $\beta $ is not very compatible with this prior form. We also find the
IBR partition using the Bayes partition for the effective RLPP, which merges
uncertainty in $\theta $ and $\rho $. In particular, we compute the
partition error for all partitions of the images from~%
\eqref{eq:error_of_partition:vectors}, and choose the partition with minimal
partition error. Note that~\eqref{eq:error_of_partition:vectors} is found
using the natural cost function in~\eqref{eq:cost_function:final}, and the
posterior partition probabilities in~\eqref{eq:partition_probability}, which
is based on posterior label function probabilities that may be computed
exactly using a discretized version of~%
\eqref{eq:image_label_function_integral}. Recall $f(\mathbf{x}|y,\rho
,\theta )$ is assumed Gaussian with means given by~\eqref{eq:mu},
covariances given by~\eqref{eq:Sigma}, and appropriate values for $b_{1}$, $%
b_{2}$ and $\beta $ depending on $y$, $\rho $ and $\theta $. It is possible
to list all partitions and compute the partition errors exactly when $n=10$
and $l=2$. Again, we did not test MCBR and minimax robust clusterers owing
to their high computational cost. 

Figure~\ref{fig:performance_aligned} shows the approximate clustering error
for all algorithms with respect to $\theta$, computed using the average
cluster mismatch error over $500$ sets of images for each $\theta$. Part~(a)
shows results when $n_{1}=n_{2}=5$, and part~(b) shows results when $n_{1}=6$
and $n_{2}=4$. In both parts~(a) and~(b), the IBR clusterer performs much
better than all classical algorithms across all states. Note that the IBR
clusterer makes ``incorrect'' Gaussian modeling assumptions, but that the
Gaussianity assumption and the analytically computed mean and covariance for
each cluster become more accurate as the number of grains increases. Under
all algorithms there is a significant variation in performance, which
deteriorates when $\theta \approx 1.8750$. This corresponds to the case
where $\beta _{1}=\beta _{2}$, i.e., the case where the classes are most
similar. Among all classical algorithms, the EM algorithm is usually the
best, followed by FCM and KM, which have very similar performance. In some
cases in Figure~\ref{fig:performance_aligned}, the performance of
hierarchical clustering with single linkage is worse than Random. As seen in
Section~\ref{sec:simulation_Gaussian}, MCMC with incorrect priors and small
samples again has very poor performance. These graphs are summarized in
Table~\ref{tab:granulometry}, which shows the approximate average clustering
error for each algorithm over all states and iterations. Finally, note that
performance is similar between equal and unequal cluster size for all
algorithms.

\begin{figure*}[tb!]
\centering
\subfloat[]{\includegraphics[width=0.48\textwidth]{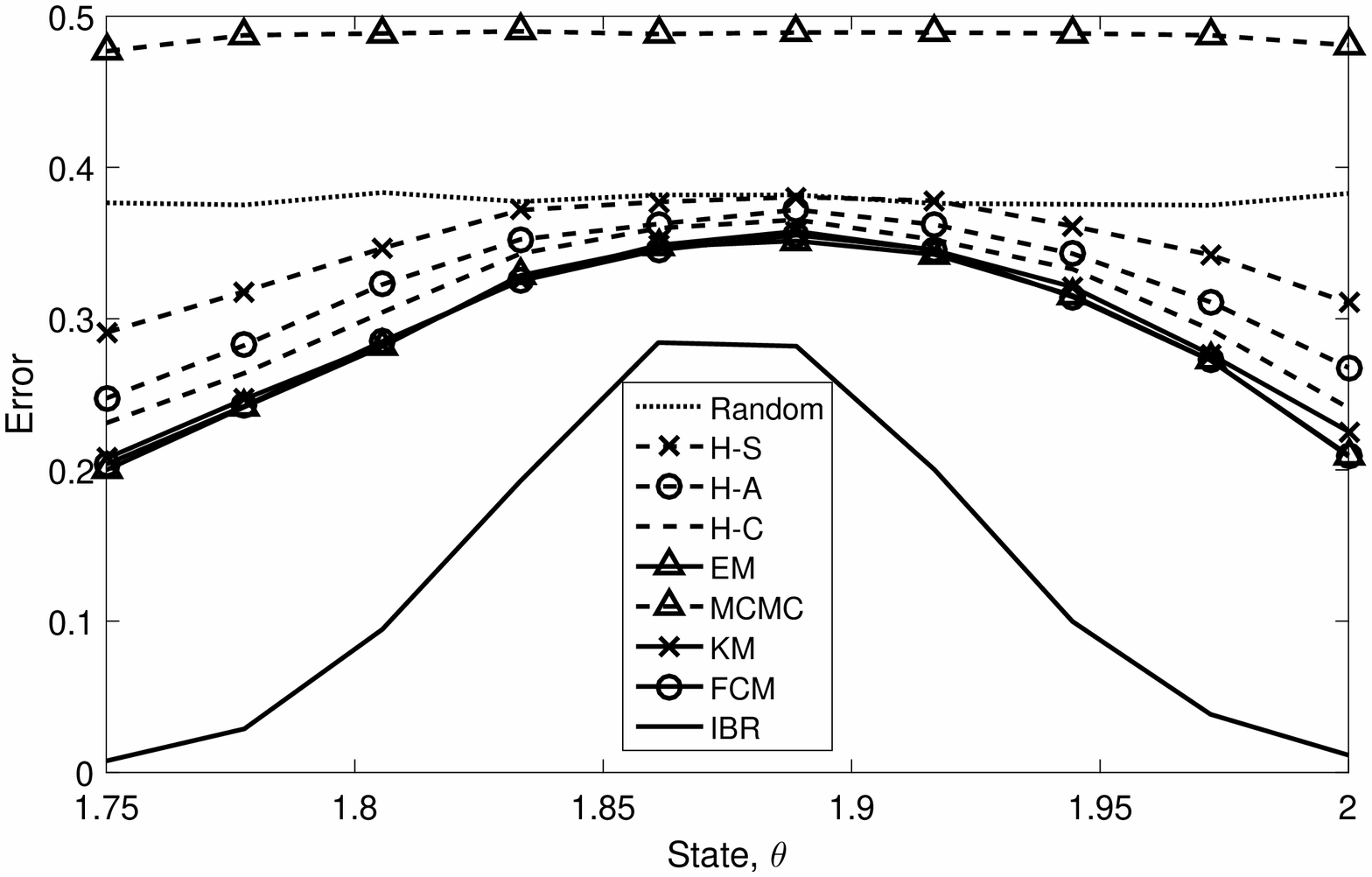}}
\hfill
\subfloat[]{\includegraphics[width=0.48\textwidth]{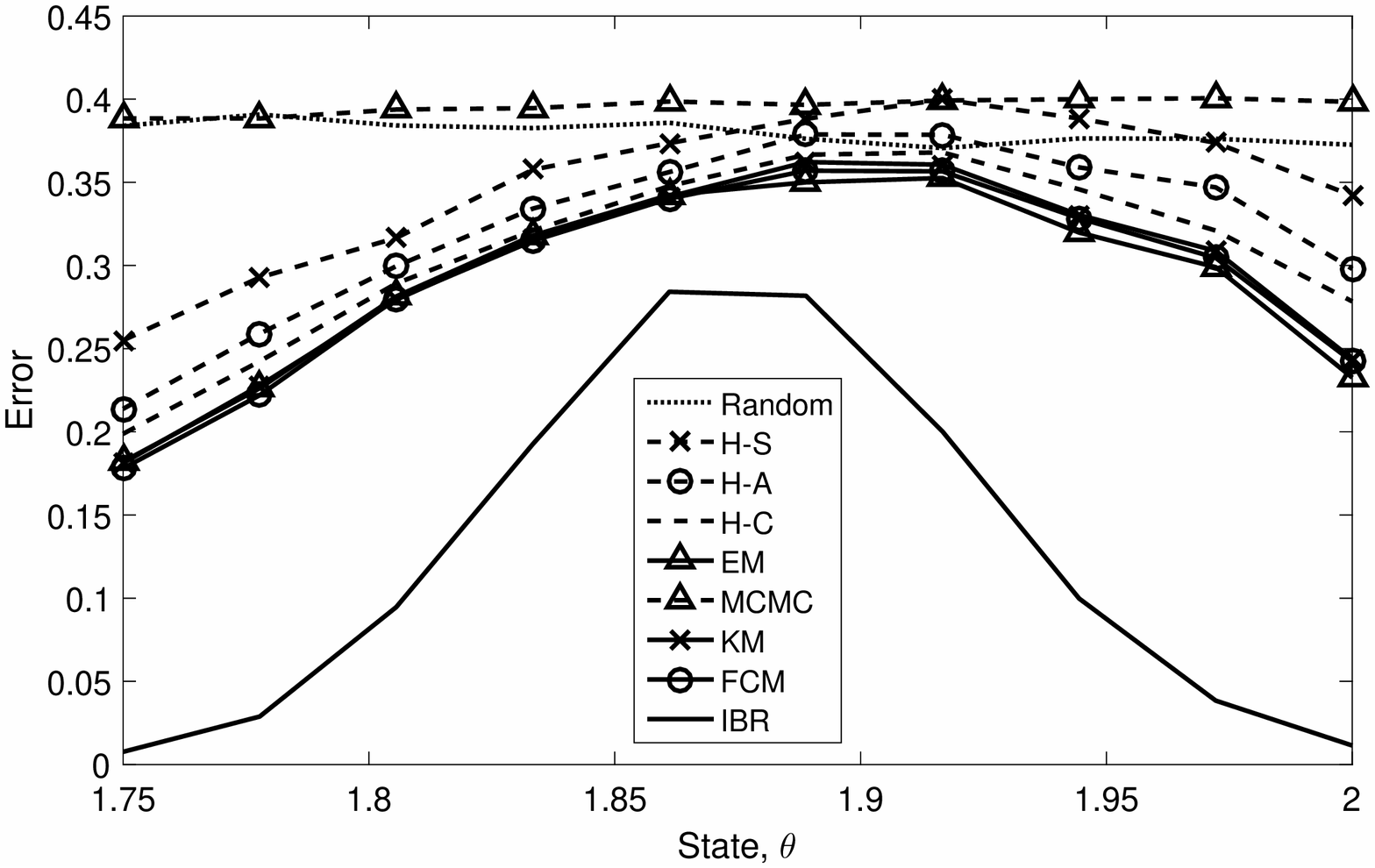}}
\caption{Average cluster mismatch error as a function of the state, $\protect%
\theta$, in the granular imaging example: (a) $n_1 = n_2 = 5$, (b) $n_1 = 6$
and $n_2 = 4$. }
\label{fig:performance_aligned}
\end{figure*}

\begin{table*}[!tb]
\caption{Average cluster mismatch error over all $5,000$ iterations in the
granular imaging example. }
\label{tab:granulometry}\centering
\begin{tabular}{|c|rrrrrrrrr|}
\hline
$n_1, n_2$ & IBR & FCM & KM & MCMC & EM & H-C & H-A & H-S & Random \\ 
\hline
5, 5 & 0.1239 & 0.2899 & 0.2938 & 0.4858 & 0.2890 & 0.3086 & 0.3223 & 0.3477
& 0.3786 \\ 
6, 4 & 0.1351 & 0.2924 & 0.2952 & 0.3956 & 0.2904 & 0.3079 & 0.3224 & 0.3488
& 0.3799 \\ 
\hline
\end{tabular}%
\end{table*}

Since our focus is on robust clustering theory rather than image processing,
in particular, the interplay between clustering optimization and the
structure of prior knowledge, we have implemented a model setting based on~%
\eqref{eq:image_model}; nevertheless, before concluding this section, we
believe a few comments concerning the effect of deviations from the model
assumptions on the asymptotic granulometric moments are warranted.

The grain model of~\eqref{eq:image_model} has been used in numerous studies
of granulometric filtering and asymptotic moment analysis. Three issues
regarding robustness of the theory to deviations from model assumptions have
been addressed in~\cite{rotation_practice}: (1) assuming a certain sizing
distribution when in fact the random set satisfies a different sizing
distribution, (2) using erroneous parameters for the sizing distribution,
and (3) prior segmentation when there is modest overlapping.

For instance, the effect of erroneous gamma$(\alpha ,\beta )$ sizing was
analytically quantified with respect to misclassification error. Perhaps
more importantly, the effect of watershed segmentation to separate
overlapping grains prior to moment analysis was quantified by establishing
lower and upper bounds on the actual $k$th granulometric moments when there
are multiple grain primitives. One could reconsider the entire clustering
analysis relative to these bounds; however, given the complexity of the
bounds, this would involve a complicated mathematical study that would lead
us far afield. The bounds are quite tight when grain overlapping is minor,
as it is with a properly prepared emulsion.

Finally, as in all asymptotic granulometric theory, grain orientation is
assumed fixed and not subject to rotation. The assumption is that each grain
can be canonically rotated so that triangles have a horizontal base and for
rods the shorter side forms the base, as assumed in the model. Robustness
relative to imperfect rotation normalization has not been studied
analytically. In fact, in digital image processing, rotation can cause
problems for triangles and rectangles when edge detection is inaccurate,
which is troublesome when there is low pixel resolution, a situation that is
much less problematic today than when the basic granulometric theory was
developed twenty years ago.

\section{Conclusion}

We have extended the theories of robust filtering and classification to
clustering and developed new theory showing that optimal Bayesian robust
clustering can be viewed as two equivalent optimization problems, one based
on a parameterized uncertainty class of RLPPs and the other on a single
effective RLPP that absorbs all parameters in the model. Thus, one can first
focus on modeling the uncertainties and then focus on finding the Bayes
clusterer (or a good approximation) for the effective model.

The proposed paradigm for robust clustering is distinct from all other
clustering methods in that it is fully model-based, can account for all
prior knowledge and sources of uncertainty, and is optimal relative to
clustering error. A key part of the paradigm involves justifying the
modeling assumptions. In cases where the modeling assumptions can be
justified, like in our granular imaging example where we developed new
theory on the asymptotic joint normality and moments of our extracted
features, we now have a very powerful theory for optimal robust clustering.
Furthermore, since the Bayes and IBR clusterers employ powerful optimization
directly with respect to clustering error (or clustering risk if used with
specialized cost functions), under small to moderate imperfections of the
assumed model they often continue to outperform many principled
optimization-based methods. For instance, although our implementations of
the EM, MCMC and IBR algorithms all assume Gaussian mixture models, EM and
MCMC do not always perform as well as IBR because: (1) they focus on
estimating the means and covariances instead of minimizing error, (2) they
are often implemented without available prior knowledge.

We conclude with a note on computational complexity. The optimal IBR
clusterer under Gaussian models is computationally expensive, which remains
an important issue. Since our objective here has been to develop a theory of
robust clustering, we have focused on clustering a small number of points
and implemented optimal algorithms whenever possible. That being said,
suboptimal methods inspired by the optimal equations for the Bayes clusterer
under Gaussian models (e.g., Suboptimal Pseed Fast) have been presented in~%
\cite{dalton2015analytic}. These have nearly optimal performance and
competitive computation time with point sets of size up to $10,000$. We aim
to continue studying fast suboptimal algorithms for the Bayes clusterer in
future work, which by Theorems~\ref{theorem:evaluating_IBR} and~\ref%
{theorem:effective_RLPP} automatically extend to algorithms for robust
clustering.

\section*{Appendix}

Here, we justify modeling assumptions like normality used by the IBR
clusterer in the granular imaging example. The following theorems,
originally proved in~\cite{Sand:98} and~\cite{Sivakumar:01}, provide exact
expressions for granulometric moments as a function of the grain radii. They
state that any finite length vector of granulometric moments from a single
structuring element is asymptotically normal, and provide analytic
expressions for the asymptotic mean and variance of moments. The covariance
of moments is available in~\cite{Sivakumar:01}.

\begin{thm}
Let $I$ be modeled as in~\eqref{eq:image_model}. For the granulometry $\{I
\circ tB\}$ generated by a convex, compact structuring element $B$, and for $%
k \geq 1$, 
\begin{equation}
\mu^{(k)}(I, B) = \frac{u}{v} \equiv H(u, v),
\label{eq:granulometric_moment}
\end{equation}
where 
\begin{align}
u &= \frac{1}{N} \sum_{i=1}^{m} \sum_{j=1}^{N_i} \mu^{(k)}(A_i, B) \nu[A_i]
r_{ij}^{k+2},  \label{eq:granulometric_moment_u}
\end{align}
\begin{align}
v &= \frac{1}{N} \sum_{i=1}^{m} \sum_{j=1}^{N_i} \nu[A_i] r_{ij}^{2},
\label{eq:granulometric_moment_v}
\end{align}
$\nu[A_i]$ is the volume of $A_i$, and $\mu^{(k)}(A_i, B)$ is the $k$th
moment of $A_i$ under structuring element $B$. Moreover, suppose:
\begin{enumerate}
\item The proportions $b_i=N_i/N$ are known and fixed.

\item The $r_{ij}$ are independent, $r_{i1}, \ldots, r_{i N_i}$ are
identically distributed, and every $r_{ij}$ has finite moments up to at
least order $k + 2$.

\item There exist $c$ and $t > 0$ such that $H \leq c N^t$ for $N > 1$.

\item $H$ has first and second derivatives, with bounded second derivatives
in a neighborhood of $(E[u], E[v])$.
\end{enumerate}
Then the distribution of $H$ is asymptotically normal as $N \to \infty$ with
mean and variance given by 
\begin{align}
E[H] &= H(E[u], E[v]) + O(N^{-1})  \label{eq:H_mean}
\end{align}
and 
\begin{align}
&Var[H] = (\textstyle \frac{\partial H}{\partial u} (E[u], E[v]))^2 Var[u] 
\notag \\
&\quad + 2 \textstyle \frac{\partial H}{\partial u} (E[u], E[v]) \textstyle 
\frac{\partial H}{\partial v} (E[u], E[v]) Cov[u, v]  \notag \\
&\quad + (\textstyle \frac{\partial H}{\partial v} (E[u], E[v]))^2 Var[v] +
O(N^{-3/2}).  \label{eq:H_var}
\end{align}
\label{thm:Sand98}
\end{thm}

\begin{thm}
Under the conditions of Theorem~\ref{thm:Sand98}, any finite set of
granulometric moments is asymptotically jointly normal. \label%
{thm:Sivakumar:01}
\end{thm}

Theorems~\ref{thm:Sand98} and~\ref{thm:Sivakumar:01} are not sufficient to
guarantee the asymptotic joint normality of $\mathbf{x}$ or $\mathbf{z}$, or
to obtain their asymptotic moments, because these vectors contain moments
from multiple structuring elements. 
Thus, here we present a new lemma proving asymptotic joint normality and
providing analytic expressions for the asymptotic mean and covariance of
granulometric moments under multiple primitives and multiple structuring
elements. It can be shown that these moments are consistent with~%
\eqref{eq:H_mean} and~\eqref{eq:H_var}.

\begin{lem}
Let $I$ be modeled as in~\eqref{eq:image_model} with $m = 2$ primitives, let 
$b_1 = N_1/N$ and $b_2 = N_2/N$ be known and fixed, and let the radii $%
r_{ij} $ be independent such that $r_{i1}, \ldots, r_{i N_i}$ are
identically distributed and $E[r_{ij}^k] = \gamma_{ik} \beta^k$ for $k = 2,
3, 4$. Let $\mathbf{x} = M^{-1} \mathbf{z}$ be a vector of linearly
transformed first and second order granulometric moments under arbitrary
structuring elements, $B_1$ and $B_2$. 
%
Then $\mathbf{x}$ is asymptotically jointly normal with mean vector 
\begin{equation}
\frac{1}{b_1 \gamma_{12} + b_2 \gamma_{22}} 
\begin{bmatrix}
b_1 \gamma_{13} \beta & b_2 \gamma_{23} \beta & b_1 \gamma_{14} \beta^2 & 
b_2 \gamma_{24} \beta^2%
\end{bmatrix}%
^T  \label{eq:x_mean}
\end{equation}
and covariance matrix 
\begin{equation}
\frac{1}{N(b_1 \gamma_{12} + b_2 \gamma_{22})^4} 
\begin{bmatrix}
A_{11} \beta^2 & A_{12} \beta^3 \\ 
A_{21} \beta^3 & A_{22} \beta^4%
\end{bmatrix}%
,  \label{eq:x_cov}
\end{equation}
where 
\begin{align*}
A_{ij} &= 
\begin{bmatrix}
(b_1)^2 (b_1 B_{ij1}^{11} + b_2 B_{ij1}^{12}) & b_1 b_2 (b_1 B_{ij1}^{21} +
b_2 B_{ij1}^{22}) \\ 
b_1 b_2 (b_1 B_{ij2}^{11} + b_2 B_{ij2}^{12}) & (b_2)^2 (b_1 B_{ij2}^{21} +
b_2 B_{ij2}^{22})%
\end{bmatrix}
\notag \\
& + 
\begin{bmatrix}
b_1 (b_1 \gamma_{12} + b_2 \gamma_{22})^2 C_{ij1} & 0 \\ 
0 & b_2 (b_1 \gamma_{12} + b_2 \gamma_{22})^2 C_{ij2}%
\end{bmatrix}%
, \\
B_{ijk}^{lp} &= C_{00p}\gamma_{k(i+2)}\gamma_{l(j+2)} \\
&\quad - \gamma_{p2}\gamma_{k(i+2)}C_{0jl} -
\gamma_{p2}C_{i0k}\gamma_{l(j+2)}, \\
C_{ijk} &= \gamma_{k(i+j+4)} - \gamma_{k(i+2)}\gamma_{k(j+2)}.
\end{align*}
\label{lem:granulometry}
\end{lem}

\begin{proof}
Note $\mathbf{x} = 
\begin{bmatrix}
	x_{11} & x_{21} & x_{12} & x_{22}
\end{bmatrix}^{T}$ has components $x_{ik}$ given by~\eqref{eq:x}.  
Without loss of generality, assume each primitive has area $\nu[A_1] = \nu[A_2] = 1$. 
Fix $i$, and for each $j = 1, \ldots, N_i$ let $\mathbf{r}_{ij} = 
\begin{bmatrix}
r_{ij}^2 & r_{ij}^3 & r_{ij}^3
\end{bmatrix}
^T$. $\mathbf{r}_{ij}$ is thus a sequence of independent and identically distributed random vectors, and let
us denote the common mean vector by $\mu_i = 
\begin{bmatrix}
\mu_{i2} & \mu_{i3} & \mu_{i4}
\end{bmatrix}
^T$ and covariance matrix by $\Sigma_i$, which are assumed to exist. Then by
the central limit theorem, 
\begin{equation*}
\sqrt{N_i} \left(\mathbf{r}_i - \mu_i\right) \longrightarrow \mathcal{N}(\mathbf{0}, \Sigma_i),
\end{equation*}
where $\mathbf{r}_i$ is the sample mean of the $\mathbf{r}_{ij}$ over all $j$, 
$\mathcal{N}$ denotes a multivariate Gaussian distribution, and ``$\longrightarrow$'' denotes convergence in distribution. Since the sets of radii $\{r_{1j}\}$ and $\{r_{2j}\}$ are independent, 
$\mathbf{r}_1$ and $\mathbf{r}_2$ are independent. Hence, as long as $N_1$
and $N_2$ go to infinity as $N$ goes to infinity, 
we have for $\mathbf{w} = 
\begin{bmatrix}
b_1 \mathbf{r}_1^T & b_2 \mathbf{r}_2^T
\end{bmatrix}
^T$ and $\mu = 
\begin{bmatrix}
b_1 \mu_1^T & b_2 \mu_2^T
\end{bmatrix}^T$, 
\begin{equation*}
\sqrt{N} (\mathbf{w} - \mu) \longrightarrow \mathcal{N}\left(\mathbf{0}, 
\begin{bmatrix}
b_1 \Sigma_1 & \mathbf{0} \\
\mathbf{0} & b_2 \Sigma_2
\end{bmatrix}
\right).
\end{equation*}
Define $g:\mathbb{R}^6 \to \mathbb{R}^4$, where 
\begin{equation*}
g(\mathbf{w}) = \frac{1}{w_1 + w_4} 
\begin{bmatrix}
w_2 & w_5 & w_3 & w_6
\end{bmatrix}^T
\end{equation*}
for $\mathbf{w} = 
\begin{bmatrix}
w_1 & w_2 & w_3 & w_4 & w_5 & w_6
\end{bmatrix}^T$. $g$ is differentiable whenever $w_1 + w_4 \neq 0$, and 
\begin{align*}
&\frac{\partial g}{\partial \mathbf{w}} = \frac{1}{(w_1 + w_4)^2} \times \\
& 
\begin{bmatrix}
-w_2 & w_1 + w_4 & 0 & -w_2 & 0 & 0 \\
-w_5 & 0 & 0 & -w_5 & w_1 + w_4 & 0 \\
-w_3 & 0 & w_1 + w_4 & -w_3 & 0 & 0 \\
-w_6 & 0 & 0 & -w_6 & 0 & w_1 + w_4
\end{bmatrix}.
\end{align*}
Applying the \emph{Multivariate Delta Method}~\cite{cramer1947mathematical},
and noting that $\mu_{12}$ and $\mu_{22}$ will not both be zero for any
reasonable sizing distribution, 
\begin{align*}
&\sqrt{N} (g(\mathbf{w}) - g(\mu)) \longrightarrow \\
&\qquad \mathcal{N}\left( \mathbf{0}, \left(\frac{\partial g}{\partial 
\mathbf{w}} (\mu) \right) 
\begin{bmatrix}
b_1 \Sigma_1 & \mathbf{0} \\
\mathbf{0} & b_2 \Sigma_2
\end{bmatrix}
\left(\frac{\partial g}{\partial \mathbf{w}} (\mu) \right)^T \right).
\end{align*}
Note that $\mathbf{x} = g(\mathbf{w})$,
and that 
\begin{equation*}
\frac{\partial g}{\partial \mathbf{w}} (\mu) = \frac{1}{(b_1 \mu_{12} + b_2
\mu_{22})^2} 
\begin{bmatrix}
G_{11} & G_{12} \\
G_{21} & G_{22} \\
\end{bmatrix},
\end{equation*}
where 
\begin{align*}
G_{11} &= 
\begin{bmatrix}
-b_1 \mu_{13} & b_1\mu_{12} + b_2 \mu_{22} & 0 \\
-b_2 \mu_{23} & 0 & 0
\end{bmatrix}, \\
G_{12} &= 
\begin{bmatrix}
-b_1 \mu_{13} & 0 & 0 \\
-b_2 \mu_{23} & b_1\mu_{12} + b_2 \mu_{22} & 0
\end{bmatrix}, \\
G_{21} &= 
\begin{bmatrix}
-b_1 \mu_{14} & 0 & b_1\mu_{12} + b_2 \mu_{22} \\
-b_2 \mu_{24} & 0 & 0
\end{bmatrix}, \\
G_{22} &= 
\begin{bmatrix}
-b_1 \mu_{14} & 0 & 0 \\
-b_2 \mu_{24} & 0 & b_1\mu_{12} + b_2 \mu_{22}
\end{bmatrix}.
\end{align*}
Assuming that $\mu_{ik} = \gamma_{ik} \beta^k$, $\mathbf{x}$ is thus
asymptotically Gaussian with mean vector given by~\eqref{eq:x_mean} and covariance matrix given by~\eqref{eq:x_cov}.
Note $A_{ij}$ is a $2\times2$ matrix and $B_{ijk}^{lp}$ and $C_{ijk}$ are
constants. 
\end{proof}

\section*{Acknowledgment}

The work of LAD is supported by the National Science Foundation (CCF-1422631
and CCF-1453563).


\bibliographystyle{IEEEtran}



\end{document}